\documentclass{article}

\usepackage{arxiv}

\usepackage[utf8]{inputenc} 
\usepackage[T1]{fontenc}    
\usepackage{hyperref}       
\usepackage{url}            
\usepackage{booktabs}       
\usepackage{amsfonts}       
\usepackage{nicefrac}       
\usepackage{microtype}      
\usepackage[pdftex]{graphicx}
\graphicspath{{./png/}{./pdf/}}
\usepackage{amsmath}
\usepackage{amssymb}
\usepackage{tcolorbox}
\usepackage{soul}
\usepackage{amsthm}
\usepackage{multirow}

\usepackage{mathtools}
\usepackage[
backend=biber,
sorting=none
]{biblatex}
\newtheorem{theorem}{Theorem}[section]
\newtheorem{corollary}{Corollary}[theorem]
\newtheorem{lemma}[theorem]{Lemma}

\DeclarePairedDelimiterX{\norm}[1]{\lVert}{\rVert}{#1}

\hypersetup{
    colorlinks=true,
    linkcolor=blue,
    filecolor=blue,      
    urlcolor=blue,
}

\title{Gradient-based Competitive Learning: Theory}

\author{
 Giansalvo Cirrincione*\hspace{1mm}\href{https://orcid.org/0000-0002-2894-4164}{\includegraphics[scale=0.06]{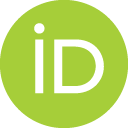}}\\
  University of Picardie Jules Verne\\
  France\\
  University of South Pacific\\
  Fiji
  \And
  Pietro~Barbiero*\hspace{1mm}\href{https://orcid.org/0000-0003-3155-2564}{\includegraphics[scale=0.06]{orcid.png}}\\
  Cambridge University\\
  United Kingdom\\
  \texttt{barbiero@tutanota.com} \\
   \And
 Gabriele Ciravegna\hspace{1mm}\href{https://orcid.org/0000-0002-6799-1043}{\includegraphics[scale=0.06]{orcid.png}}\\
  University of Florence\\
  Italy\\
   \And
 Vincenzo Randazzo\hspace{1mm}\href{https://orcid.org/0000-0003-3640-8561}{\includegraphics[scale=0.06]{orcid.png}}\\
  Politecnico di Torino\\
  Italy\\
}

\begin{document}
\maketitle

{\let\thefootnote\relax\footnote{*Equal contribution.}}
{\let\thefootnote\relax\footnote{
P. Barbiero conceived the main idea and the dual method. G. Cirrincione developed the theory and validated theorems, assumptions, and proofs with P. Barbiero. G. Ciravegna and P. Barbiero conceived, planned, and carried out the experiments. All authors discussed the results and contributed to the final manuscript.}}

\begin{abstract}
Deep learning has been widely used for supervised learning and classification/regression problems. Recently, a novel area of research has applied this paradigm to unsupervised tasks; indeed, a gradient-based approach extracts, efficiently and autonomously, the relevant features for handling input data. However, state-of-the-art techniques focus mostly on algorithmic efficiency and accuracy rather than mimic the input manifold. On the contrary, competitive learning is a powerful tool for replicating the input distribution topology. This paper introduces a novel perspective in this area by combining these two techniques: unsupervised gradient-based and competitive learning.
The theory is based on the intuition that neural networks are able to learn topological structures by working directly on the transpose of the input matrix.
At this purpose, the vanilla competitive layer and its dual are presented. The former is just an adaptation of a standard competitive layer for deep clustering, while the latter is trained on the transposed matrix.
Their equivalence is extensively proven both theoretically and experimentally. 
However, the dual layer is better suited for handling very high-dimensional datasets.  
The proposed approach has a great potential as it can be generalized to a vast selection of topological learning tasks, such as non-stationary and hierarchical clustering; furthermore, it can also be integrated within more complex architectures such as autoencoders and generative adversarial networks.
\end{abstract}



\section{Clustering as a learning problem}
Machine learning can be generally referred as extracting information from noisy data. Depending on the paradigm, either unsupervised or supervised, this problem is called clustering or classification, respectively. Both groups of techniques can be seen as an optimization problem where a loss function is minimized. The oldest and most famous clustering technique is k-means \cite{macqueen1967some}, which iteratively adapts cluster centroid positions in order to minimize the quantization error.
This technique has been extensively used and studied to uncover unknown relations in unsupervised problems. However, its main drawback is the definition of the number of cluster centroids (\textit{k}) beforehand. This is the same issue as other famous techniques such as Gaussian Mixture Models (GMM) \cite{mclachlan1988mixture} and Neural Gas (NG) \cite{martinetz1991neural}. To overcome this limitation, several \textit{incremental} algorithms have been proposed in literature, where the number of neurons is not fixed but changes over time w.r.t the complexity of the problem at hand. This approach adds a novel unit whether certain conditions are met, e.g. the quantization error is too high or data is too far from the existing neurons; in this sense, the new unit should yield a better quantization of the input distribution. 
Some examples are the the adapative k-means \cite{bhatia2004adaptive} and the Density Based Spatial Clustering (DBSCAN) \cite{ester1996density}. 
Furthermore, unsupervised learning is generally capable of finding groups of samples that are similar under a specific metric, e.g. Euclidean distance. However, it cannot infer the underlying data topology. At this purpose, to define a local topology, the \textit{Competitive Hebbian Learning (CHL)} paradigm \cite{hebb2005organization,martinetz1993competitive,chl} is employed by some algorithms such as Self-Organizing-Map (SOM) by Kohonen \cite{kohonen1982self}, the Growing Neural Gas (GNG) \cite{fritzke1995growing} and its variants \cite{fritzke1997self,ghng,barbiero2017neural,cirrincione2020gh}. Indeed, given an input sample, the two closest neurons, called first and second winners, are linked by an edge, which locally models the input shape.  

All the previously cited techniques suffer from the curse of dimensionality \cite{altman2018curse}. 
Distance-based similarity measures are not effective when dealing with highly dimensional data (e.g. images or gene expression) \cite{barbiero2020modeling}. Therefore, many methods to reduce input dimensionality and to select the most important features have been employed, such as Principal Component Analysis (PCA) \cite{pearson1901liii} and kernel functions \cite{scholkopf1997kernel}. To better preserve local topology in the reduced space, the Curvilinear Component Analysis (CCA) \cite{demartines1997curvilinear} and its online incremental version, the GCCA \cite{cirrincione2018growing,cirrincione2018growing2}, proposed a non-linear projection algorithm. This approach is quite useful for noise removal and when input features are highly correlated, because projection reduces the problem complexity; on the contrary, when features are statistically independent, a smaller space implies worse clustering performance due to the information loss.
An alternative way for dealing with high dimensional data is the use of Deep Neural Networks (DNN). Indeed, Convolutional Neural Networks (CNN) \cite{lecun1989backpropagation} have proven to be a valid tool for handling high dimensional input distribution in case of supervised learning. The strength of CNNs relies on the convolutional filters, whose output is linearly separable. In this sense, CNN filters can also be exploited for clustering. Also, DNNs can be trained by optimizing a clustering loss function \cite{hu2017learning}, \cite{yang2016joint}, \cite{chang2017deep}. A straightforward approach, however, may lead to overfitting, where data are mapped to compact clusters that do not correspond to data topology. To overcome this problem, weight regularization, data augmentation and supervised network pre-training have been proposed \cite{min2018survey}. The latter technique exploits a pre-trained CNN (e.g. AlexNet on ImageNet \cite{krizhevsky2012imagenet}) as a feature extractor in a transfer learning way \cite{hsu2017cnn}. 
Otherwise, clustering learning procedures may be integrated with a network learning process, which allows employing more complex architectures such as  Autoencoders (AE) \cite{kramer1991nonlinear}, Variational-Autoencoders (VAE) \cite{kingma2013auto} or Generative Adversarial Networks (GAN) \cite{goodfellow2014generative}. Such techniques usually employ a two-step learning process: first, a good representation of the input space is learnt through a network loss function and, later, the quantization is fine-tuned by optimizing a clustering-specific loss. The network loss can be either the reconstruction loss of an AE, the variational loss of a VAE or the adversarial loss of a GAN.  

To our knowledge, no previous work suggested to join DNN feature transformation skill with the higher representation capabilities of competitive learning approaches. In this paper, we propose two variants of a neural architecture where competitive learning is embedded in the training loss function. These networks can also be placed on top of more complex models, such as AE, CNN, VAE or GAN.

\section{Gradient-based competitive learning} \label{sec:methods}

\begin{figure*}[!ht]
    \centering
    \includegraphics[trim=0 30 0 70, clip, width=0.9\columnwidth]{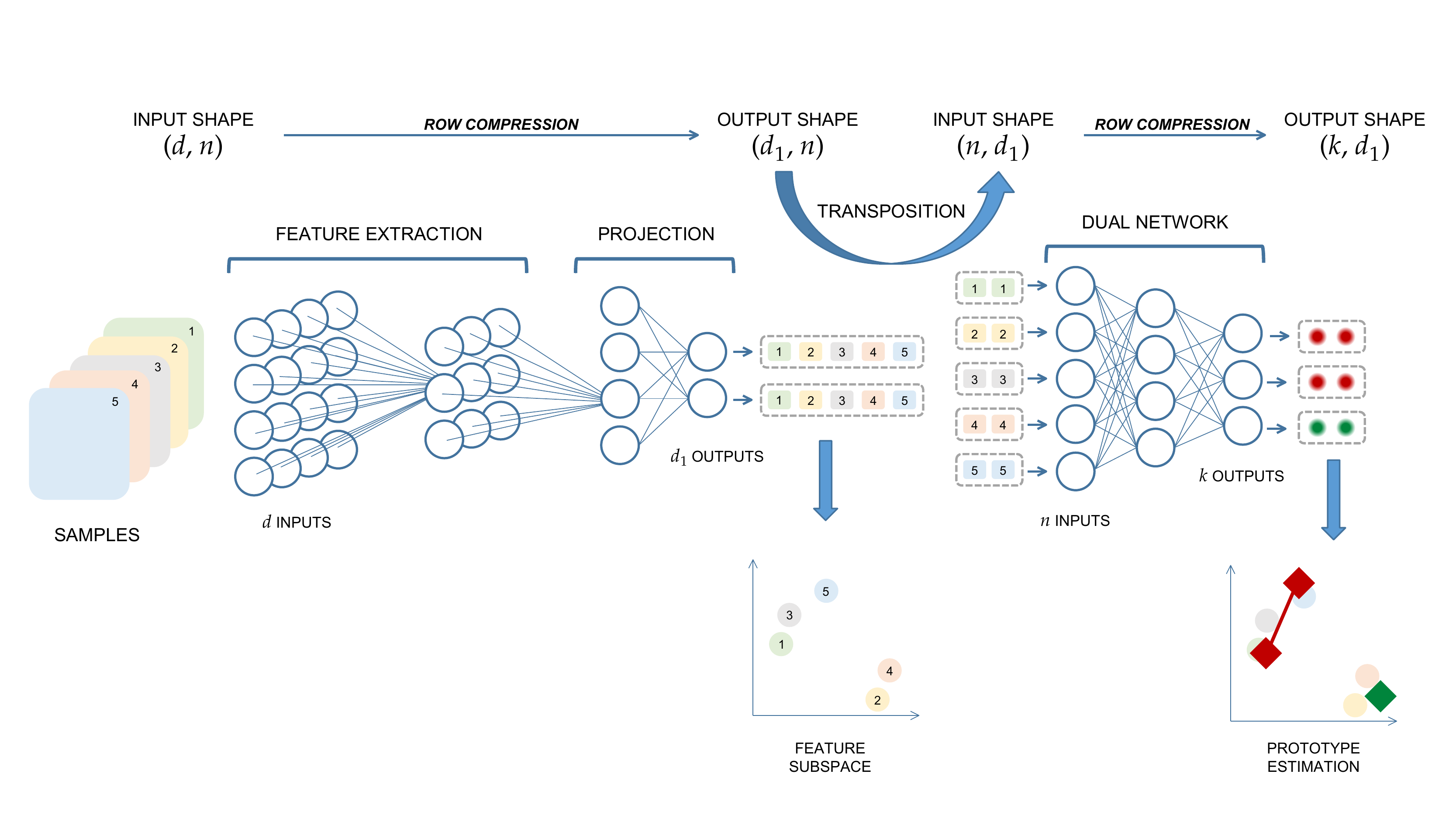}
    \caption{Representation of a deep architecture where a dual competitive network is used to estimate cluster centroids. The first network executes a feature extraction and then maps training observations into a $d_1$-dimensional feature subspace. This output is transposed and used to feed the dual network to estimate prototype positions.}
    \label{fig:deep-dual}
\end{figure*}

\subsection{Dual neural networks}
Multi-layer feedforward neural networks are universal function approximators \cite{hornik1989multilayer}. Given an input matrix $X \in \mathbb{R}^{d \times n}$ containing a collection of $n$ observations and a set of $k$ supervisions $Y \in \mathbb{R}^{k \times n}$, a neural network with $d$ input and $k$ output units can be used to approximate the target features $Y$.
The relationship between $X$ and $Y$ can be arbitrarily complex, nonetheless deep neural networks can optimize their parameters in such a way that their predictions $\widehat{Y}$ will match the target $Y$.
In supervised settings, neural networks are used to combine the information of different features (rows of $X$) in order to provide a prediction $\widehat{Y}$, which corresponds to a nonlinear projection of the observations (columns of $X$) optimized to match the target $Y$. Hence, in such scenarios, the neural network will provide one prediction for each observation $i = 1,\dots,n$.

The objective of competitive learning consists in studying the underlying structure of a manifold by means of prototypes, i.e. a set of positions in the feature space representative of the input observations. Each prototype $p_k$ is a vector in $\mathbb{R}^d$ as it lies in the same feature space of the observations. Hence, competitive learning algorithms can be described as functions mapping an input matrix $X \in \mathbb{R}^{d \times n}$ in an output matrix $\widehat{P} \in \mathbb{R}^{d \times k}$ where the $j$-th column represents the prototype $p_j$. Vanilla competitive neural networks \cite{rumelhart1985feature,barlow1989unsupervised,haykin2007neural} are composed of a set of competing neurons described by a vector of weights $p_j$, representing the position of neurons (a.k.a. \textit{prototypes}) in the input space. The inverse of the Euclidean distance between the input data $x_i$ and the weight vector $p_j$ represents the similarity between the input and the prototype. For every input vector $x_i$, the prototypes \textit{compete} with each other to see which one is the most similar to that particular input vector. By following the Competitive Hebbian Learning (CHL) rule \cite{hebb2005organization,martinetz1993competitive}, the two closest prototypes to $x_i$ are connected using an edge, representing their mutual activation. Depending on the approach, the closest prototypes to the input sample move towards it, reducing the distance between the prototype and the input. As a result, the position of the competing neurons in the input space will tend to cluster centroids of the input data. The most natural way of using a feedforward neural network for this kind of task is the transposition of the input matrix $X$ while optimizing a prototype-based loss function. This idea leads to the dual competitive layer (DCL, see Section \ref{sec:duality} and \ref{sec:analysis}), i.e. a fully connected layer trained on $X^T$, thus having $n$ input units corresponding to observations and $k$ output units corresponding to prototypes. Instead of combining different features to generate the feature subspace $\mathbb{R}^k$ where samples will be projected as for classification or regression tasks, in this case the neural network combines different samples to generate a synthetic summary of the observations, represented by a set of prototypes.
Compared with the architecture of a vanilla competitive layer (VCL) \cite{rumelhart1985feature} where prototypes correspond to the set of weight vectors of the fully connected layer, the dual approach naturally fits in a deep learning framework as the fully connected layer is actually used to apply a transformation of the input. This is not the case in a VCL, where the output is not even computed.

The DCL outputs the prototypes after one batch (epoch) by means of a linear transformation represented by its weights. At this aim, a gradient-based minimization of a loss function is used, by using the whole batch. This reminds the centroid estimation of the generalized Lloyd algorithm (k-means, \cite{lloyd1982least,sabin1986global}), which, instead, uses only the Voronoi sets. This is an important difference, because the error information can be backpropagated to the previous layer, if any, by exploiting all the observations, thus providing a relaxation of the Voronoi constraint. The underlying DCL analysis can be found in Section \ref{sec:flows}.

In order to estimate the parameters of DCL, a loss function representing the quantization error of the prototypes, the Voronoi sets are estimated by means of the Euclidean distance matrix (edm). This is the same requirement of the second iteration of the generalized Lloyd algorithm. However, the latter uses this information for directly computing the centroids. The former, instead, only yields the error to be backpropagated. The analysis and choice of the loss function is illustrated in Section \ref{sec:analysis}.

By training on the transposed input, DCL looks at observations as features and vice versa. As a consequence, increasing the number of observations $n$ (rows of $X^T$) enhances the capacity of the network, as the number of input units corresponds to $n$. Providing a higher number of features, instead, stabilizes the learning process as it expands the set of input vectors of DCL.

After training, once prototype positions have been estimated, the dual network is no longer needed. Indeed, test observations can be evaluated finding the closest prototype for each sample. This means that the amount of information required to employ this approach in production environments corresponds just to the prototype matrix $\widehat{P}$.

\subsection{Duality theory for single-layer networks} \label{sec:duality}

The intuitions outlined in the previous section can be formalized in a general theory which considers the duality properties between a linear single-layer neural network and its dual, defined as a network which learns on the transpose of the input matrix and has the same number of output neurons.

\begin{figure}[th]
    \centering
    \includegraphics[width=0.6\columnwidth, trim={100 60 100 0}, clip]{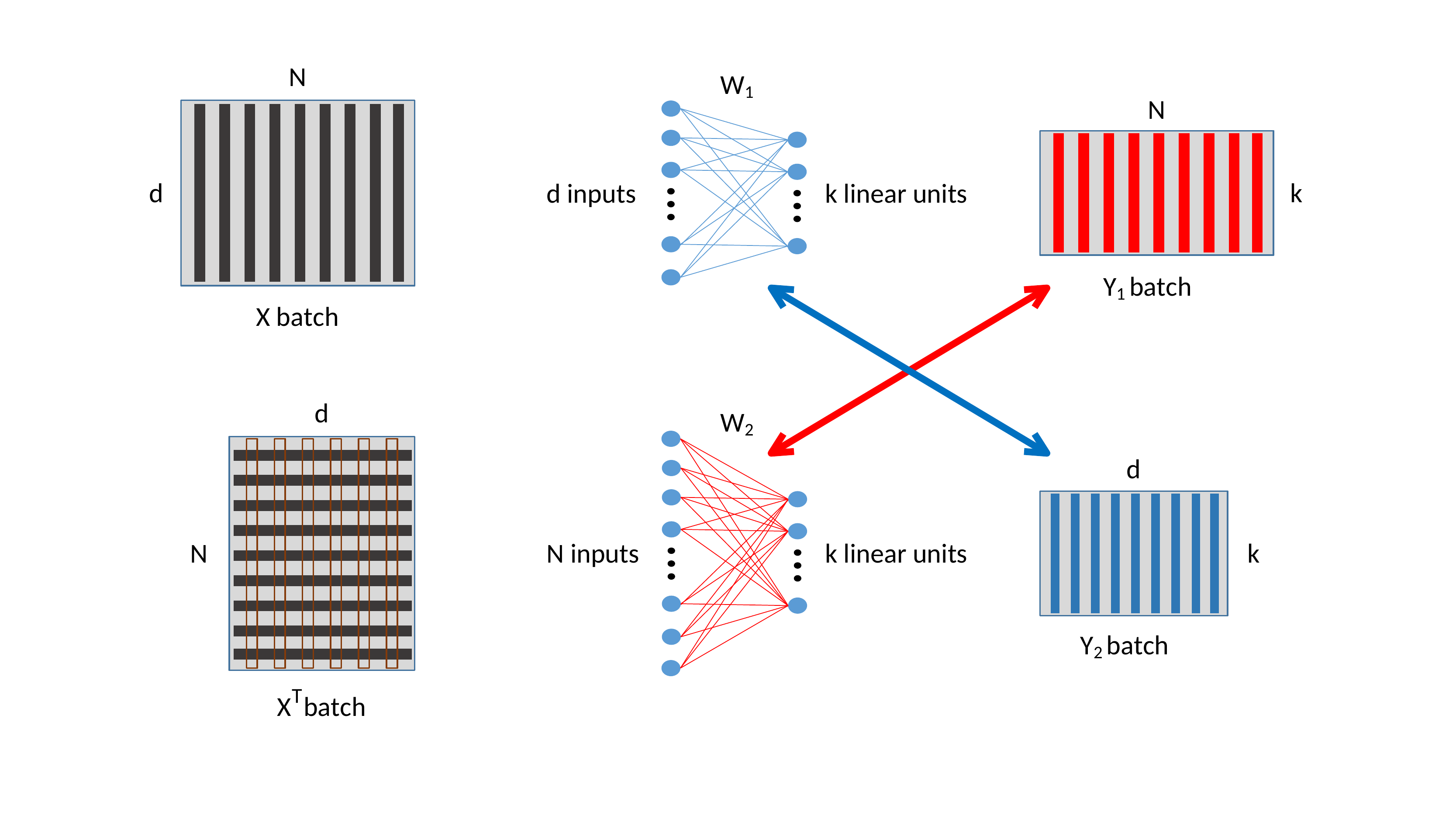}
    \caption{Base and dual single-layer neural networks.}
    \label{fig:theo-1}
\end{figure}

Consider a single layer neural network whose outputs have linear activation functions. There are $d$ input units and $k$ output units which represent a continuous signal in case of regression or class membership (posterior probabilities in case of cross entropy error function) in case of classification. A batch of $n$ samples, say $X$, is fed to the network. The weight matrix is $W_1$, where the  element $w_{ij}$ represents the weight from the input unit $j$ to the neuron $i$. 
The single layer neural network with linear activation functions in the lower scheme is here called the dual network of the former one. It has the same number of outputs and $n$ inputs. It is trained on the transpose of the original $X$ database. Its weight matrix is $W_2$ and the output batch is $Y_2$.
The following theorems state the duality conditions of the two architectures.
Figure \ref{fig:theo-1} represents the two networks and their duality.

\begin{theorem}[Network duality in competitive learning] \label{thm:duality}
Given a loss function for competitive learning based on prototypes, a single linear network (base) whose weight output neurons are the prototypes is equivalent to another (dual) whose outputs are the prototypes, under the following assumptions:
\begin{enumerate}
    \item the input matrix of the dual network is the transpose of the input matrix of the base network;
    \item the samples of the input matrix $X$ are uncorrelated with unit variance
\end{enumerate}
\end{theorem}

\begin{proof}
Consider a loss function based on prototypes, whose minimization is required for competitive learning.
From the assumption on the inputs (rows of the matrix $X$), it results $X X^T=I_d$. A single layer linear network is represented by the matrix formula:
\begin{equation}
    Y = W X = \Big[ \textrm{prototype}_1 \dots \textrm{prototype}_k \Big] X
\end{equation}
By multiplying on the right by $X^T$, it holds:
\begin{equation}
    W X X^T = Y X^T
\end{equation}
Under the second assumption:
\begin{equation}
    W = \Big[ \textrm{prototype}_1 \dots \textrm{prototype}_k \Big] = Y X^T
\end{equation}
This equation represents a (dual) linear network whose outputs are the prototypes $W$. Considering that the same loss function is used for both cases, the two networks are equivalent.
\end{proof}

This theorem directly applies to the VCL (base) and DCL (dual) neural networks if the assumption $2$ holds for the training set. If not, a preprocessing, e.g. batch normalization, can be performed.

\begin{theorem}[Impossible complete duality]
Two dual networks cannot share weights as  $W_1=Y_2$ and  $W_2=Y_1$ (complete dual constraint), except if the samples of the input matrix $X$ are uncorrelated with unit variance.
\end{theorem}

\begin{proof}
From the duality of networks and their linearity, for an entire batch it follows:
\begin{eqnarray}
    \begin{cases}
    Y_1 &= W_1 X \\
    Y_2 &= W_2 X^T
    \end{cases}
    &\implies& W_1 = Y_1 X^T \nonumber \\ 
    &\implies& W_1 = W_1 X X^T \nonumber \\ 
    &\implies& X X^T = I_d
\end{eqnarray}
\begin{eqnarray}
    \begin{cases}
    Y_1 &= W_1 X \\
    Y_2 &= W_2 X^T
    \end{cases}
    &\implies& W_2 = Y_2 X^T \nonumber \\ 
    &\implies& W_2 = W_2 X^T X \nonumber \\ 
    &\implies& X^T X = I_n
\end{eqnarray}
where $I_d$ and $I_n$ are the identity matrices of size $d$ and $n$, respectively. These two final conditions are only possible if the samples of the input matrix $X$ are uncorrelated with unit variance, which is not the case in (almost all) machine learning applications.
\end{proof}

\begin{theorem}[Half duality I]
Given two dual networks, if the samples of the input matrix $X$ are uncorrelated with unit variance and if  $W_1=Y_2$ (first dual constraint), then  $W_2=Y_1$ (second dual constraint).
\end{theorem}

\begin{proof}
From the first dual constraint (see Figure \ref{fig:theo-3}, top), for the second network it stems:
\begin{equation}
    Y_2 = W_1 = W_2 X^T
\end{equation}
Hence:
\begin{equation}
    Y_1 = W_1 X \implies Y_1 = W_2 X^T X
\end{equation}
under the second assumption on $X^T$ from Theorem \ref{thm:duality}, which implies $X^T X = I_n$, the result follows (see Figure \ref{fig:theo-3}, bottom).
\end{proof}

\begin{figure}[!ht]
    \centering
    \includegraphics[width=0.49\columnwidth,trim={8.2cm 1cm 6cm 3cm},clip]{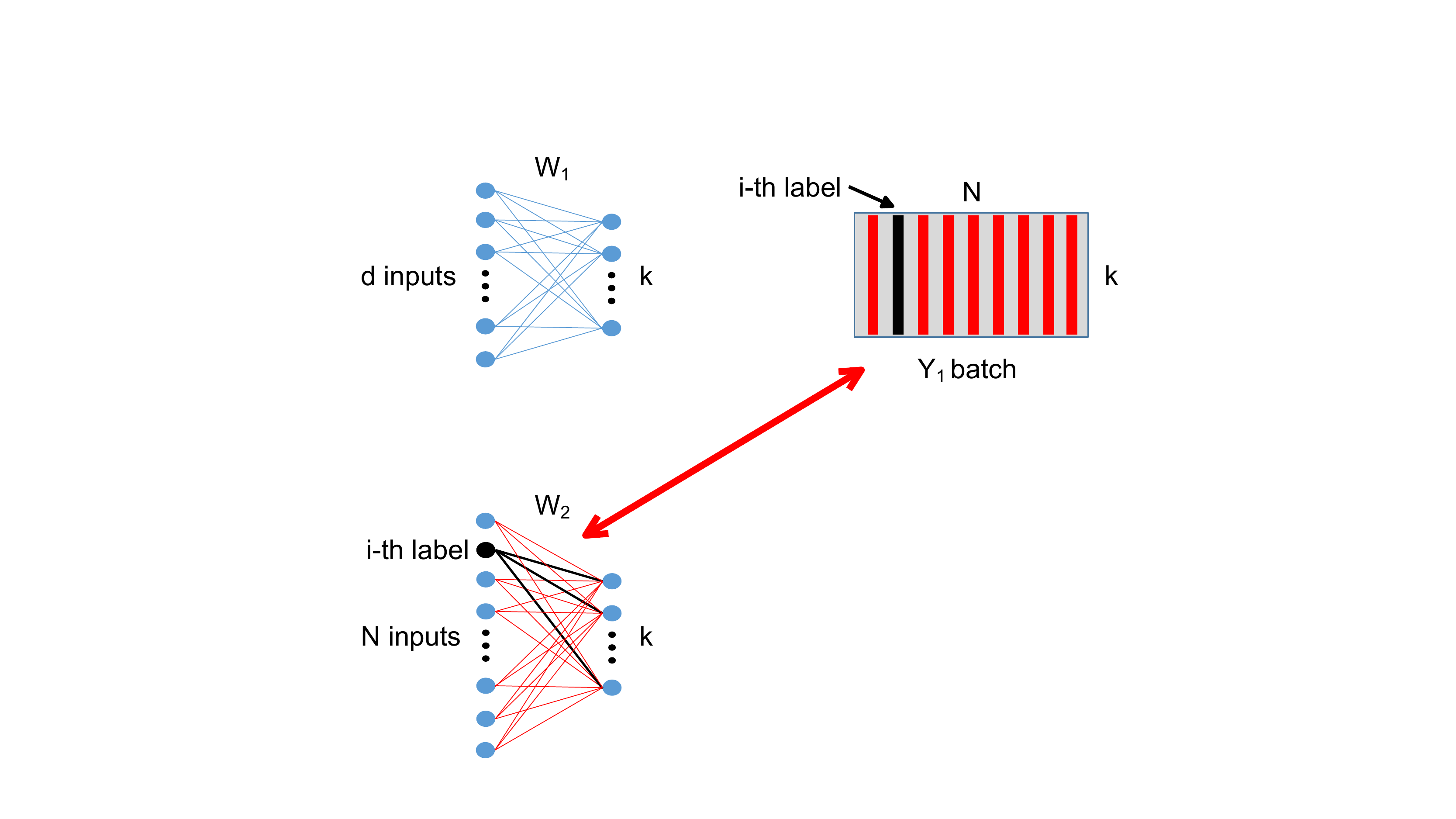}
    \includegraphics[width=0.49\columnwidth,trim={9cm 3cm 5cm 1.5cm},clip]{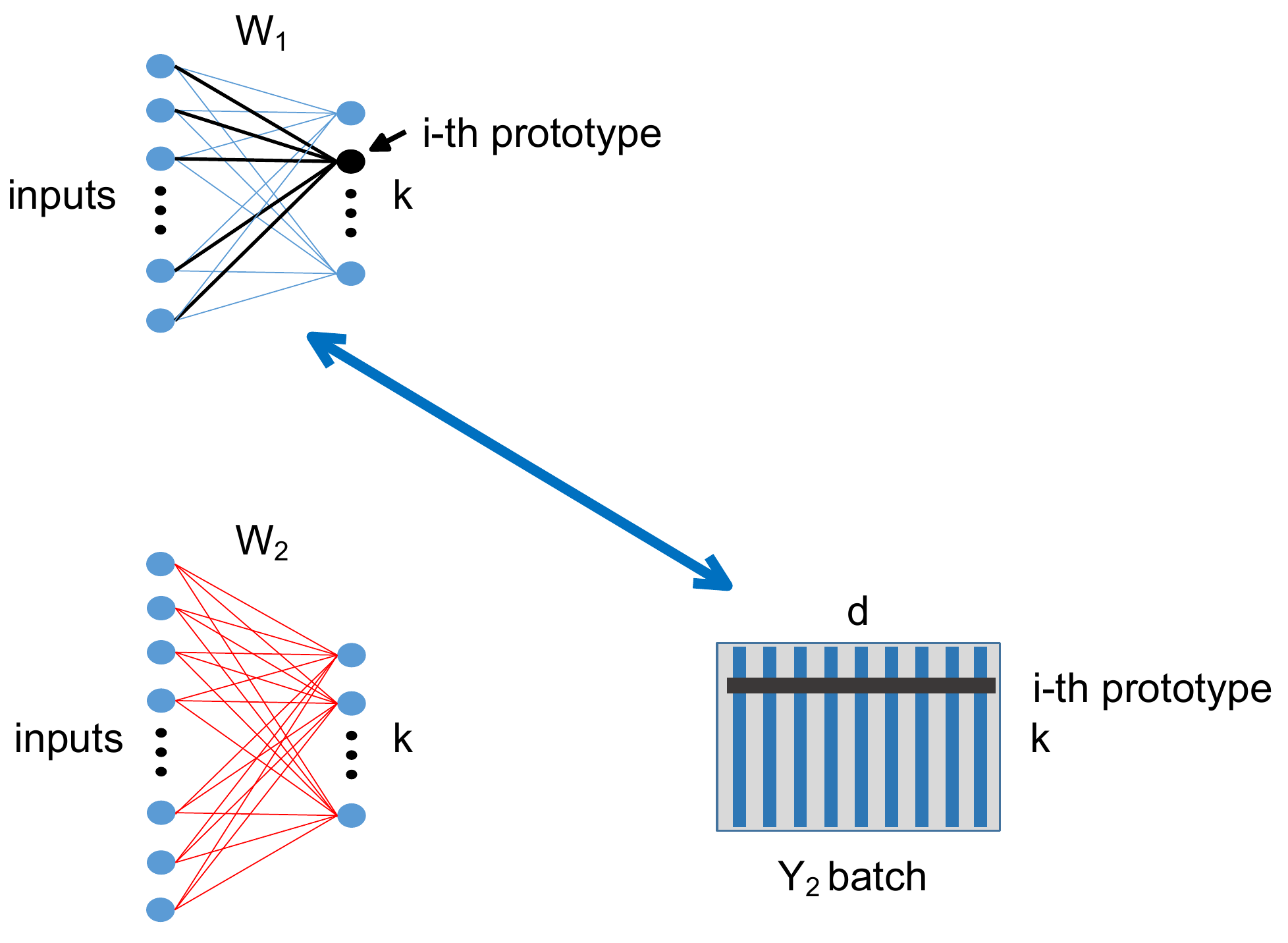}
    \caption{Half dualities.}
    \label{fig:theo-3}
\end{figure}

\begin{theorem}[Half duality II] \label{theo:HD2}
Given two dual networks, if the samples of the input matrix $X$ are uncorrelated with unit variance and if  $W_2=Y_1$ (second dual constraint), then  $W_1=Y_2$ (first dual constraint).
\end{theorem}

\begin{proof}
From the second dual constraint (see Figure \ref{fig:theo-3}, left), for the second network it stems:
\begin{equation}
    Y_1 = W_2 = W_1 X
\end{equation}
From the assumption on the inputs (rows of the matrix $X$), it results  $X X^T = I_d$. The first neural architecture yields (see Figure \ref{fig:theo-3}, right):
\begin{equation}
    Y_2 = W_2 X^T \implies Y_2 = W_1 X X^T = W_1
\end{equation}
\end{proof}
Theorem \ref{theo:HD2} justifies the use of the first single-layer neural network as a competitive layer.

\begin{corollary}[Self-supervised learning] \label{corollary:self}
The assumption of Theorem \ref{theo:HD2} implies the construction of labels for the base network.
\end{corollary}

\begin{proof}
As sketched in Figure \ref{fig:theo-3}, under the assumption of the equivalence between the training of the dual network (building of prototypes) and the architecture of the base network (output neurons as prototypes), the previous theorem implies the second dual constraint, which means the construction of a self-organized label.
\end{proof}

Thanks to this corollary, the base network can work in a self-supervised way, by using the results of the dual self-organization, to infer information on the dataset. This results in a new approach to self-supervised learning.

\subsection{Clustering as a loss minimization}
\label{sec:analysis}

The theoretical framework developed in Section \ref{sec:duality} can be easily adapted to accommodate for a variety of unsupervised learning tasks by designing a suitable loss function.
One of the most common prototype-based loss functions employed for clustering aims at minimizing the expected squared quantization error \cite{gray1984vector}.
Depending on the feature subspace, some clusters may have complex shapes, therefore using only one prototype per cluster may result in a poor representation. To overcome this limitation, each cluster can be represented by a graph composed of a collection of connected prototypes. The corresponding loss function can be written as:
\begin{equation}
\label{eq:loss}
    \mathcal{L} = \mathcal{Q} + \lambda ||E||_2
\end{equation}
where $\mathcal{Q}$ is the quantization error and $E$ is the adjacency matrix describing the connections between prototypes. 
The $\mathcal{Q}$ term is estimated from the Voronoi sets of the prototypes, which require the evaluation of the edm between $X$ and $Y$.
The $E$ term uses the CHL rule, which implies the estimation of the first and second winner w.r.t. each sample by means of the same edm.
By using the Lagrangian term $\lambda ||E||_2$, the complexity of the graph representing connections among prototypes can be minimized, in order to learn the minimal topological structure. Lonely prototypes (i.e. prototypes without connections) may represent outliers and can be easily pruned or examined individually.

The minimization of Eq. \ref{eq:loss} can be exploited for analysing the topological properties of the input manifolds. However, this is out of the scope of this paper. Indeed, it allows both the detection of clusters by means of the connectedness of the graphs and the best number of prototypes (pruning from a user-defined number of output units). This technique addresses the problem of the choice of prototypes in k-means.

Section \ref{sec:duality} established a set of conditions for the duality of two single-layer feedforward neural networks only in terms of their architecture. Instead, the choice of the learning process determines their application.
In case of clustering, they correspond to the VCL and DCL respectively, if they are both trained by the minimization of Eq. \ref{eq:loss}.
However, as it will be shown in Section \ref{sec:flows}, the equivalence in the architecture does not imply an equivalence in the training process, even if the loss function and the optimization algorithm are the same. Indeed, in a vanilla competitive layer there is no forward pass as $Y_1$ is not computed nor considered and the prototype matrix is just the weight matrix $W_1$:
\begin{equation}
    \widehat{P}_1 = \Big[\textrm{prototype}_1, \dots, \textrm{prototype}_k \Big] = W_1
\end{equation}
where $\textrm{prototype}_i \in \mathbb{R}^{d \times 1}$.
In a dual competitive layer, instead, the prototype matrix corresponds to the output $Y_2$; hence, the forward pass is a linear transformation of the input $X^T$ through the weight matrix $W_2$:
\begin{eqnarray}
\widehat{P}_2 &=& \Big[\textrm{prototype}_1 \dots \textrm{prototype}_k \Big]^T = Y_2 =  W_2 X^T = \nonumber \\
&=& \begin{bmatrix}
\mathbf{w}_1^T \\
\mathbf{w}_2^T \\
\dots \\
\mathbf{w}_k^T
\end{bmatrix}
\begin{bmatrix}
\mathbf{f}_1 & \mathbf{f}_2 & \dots & \mathbf{f}_d \\
\end{bmatrix} = \nonumber \\
&=&
\begin{bmatrix}
\mathbf{w}_1^T \mathbf{f}_1 & \mathbf{w}_1^T \mathbf{f}_2 & \dots & \mathbf{w}_1^T \mathbf{f}_d \\
\mathbf{w}_2^T \mathbf{f}_1 & \mathbf{w}_2^T \mathbf{f}_2 & \dots & \mathbf{w}_2^T \mathbf{f}_d \\
\dots & \dots & \ddots & \vdots \\
\mathbf{w}_k^T \mathbf{f}_1 & \mathbf{w}_k^T \mathbf{f}_2 & \dots & \mathbf{w}_k^T \mathbf{f}_d \\
\end{bmatrix}
\end{eqnarray}
where $\mathbf{w}_i$ is the weight vector of the $i$-th output neuron of the dual network and $\mathbf{f}_i$ is the $i$-th feature over all samples of the input matrix $X$.
The components of each prototype are computed using a constant weight $\mathbf{w}_i$, because $\widehat{P}_2$ is an outer product, which has rank $1$. Besides, each component is computed as it were a one dimensional learning problem. For instance, the first component of the prototypes is $\Big[ \mathbf{w}_1^T \mathbf{f}_1 \dots \mathbf{w}_k^T \mathbf{f}_1 \Big]^T$; which means that the first component of all the prototypes is computed by considering just the first feature $\mathbf{f}_1$. Hence, each component is independent from all the other features of the input matrix, allowing the forward pass to be just like a collection of $d$ one-dimensional problems.

Such differences in the forward pass have an impact on the backward pass as well, even if the form of the loss function is the same for both systems. However, the parameters of the optimization are not the same. For the base network:
\begin{equation}
    \mathcal{L} = \mathcal{L} (X,W_1)
\end{equation}
while for the dual network:
\begin{equation}
    \mathcal{L} = \mathcal{L} (X^T,Y)
\end{equation}
where $Y$ is a linear transformation (filter) represented by $W_2$.
In the base competitive layer the gradient of the loss function with respect to the weights $W_1$ is computed directly as:
\begin{eqnarray}
\nabla \mathcal{L} (W_1) = \frac{d\mathcal{L}}{dW_1}
\end{eqnarray}
On the other hand, in the dual competitive layer, the chain rule is required to computed the gradient with respect to the weights $W_2$ as the loss function depends on the prototypes $Y_2$:
\begin{eqnarray}
\nabla \mathcal{L} (W_2) = \frac{d\mathcal{L}}{dW_2} = \frac{d\mathcal{L}}{dY_2} \cdot \frac{dY_2}{dW_2}
\end{eqnarray}
As a result, despite the architecture of the two layers is equivalent, the learning process is quite different.

\subsection{Simulations}

\begin{figure*}[!b]
    \centering

    \includegraphics[width=0.33\columnwidth, trim= 0 80 0 0, clip]{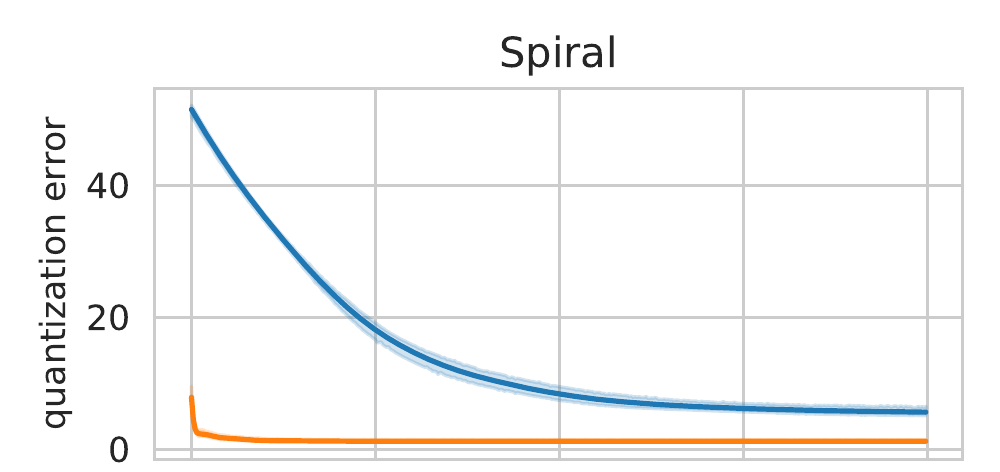}
    \includegraphics[width=0.33\columnwidth, trim= 0 80 0 0, clip]{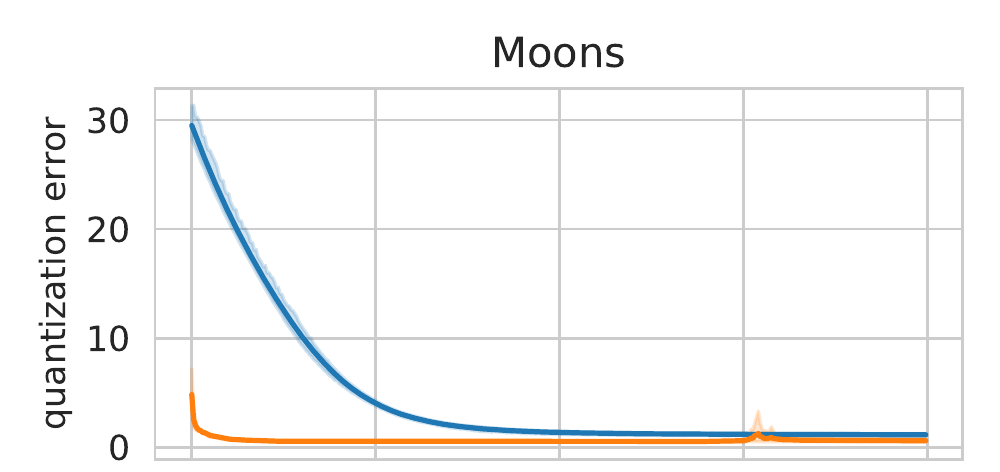}
    \includegraphics[width=0.33\columnwidth, trim= 0 80 0 0, clip]{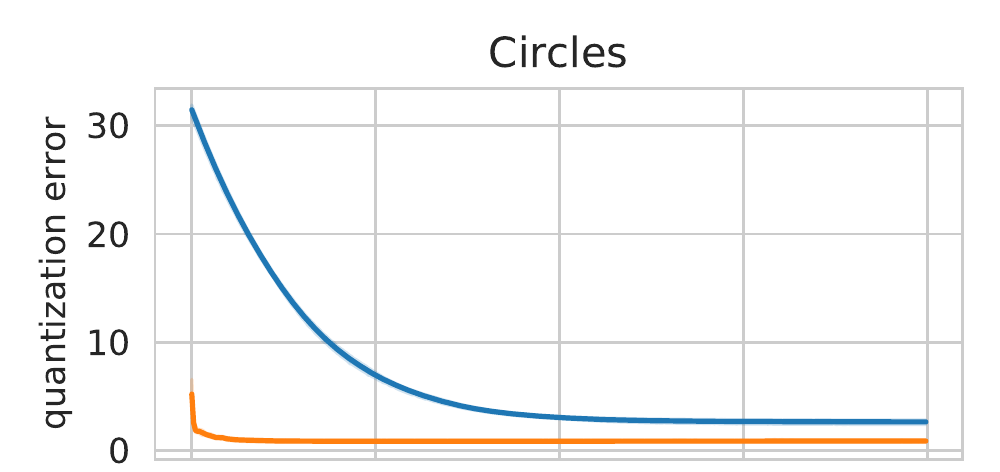}\\
    
    \includegraphics[width=0.33\columnwidth, trim= 0 80 0 0, clip]{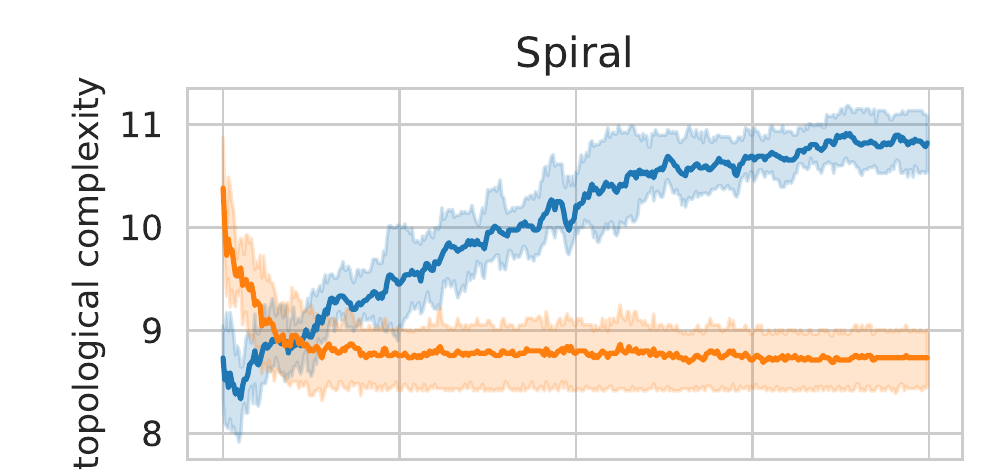}
    \includegraphics[width=0.33\columnwidth, trim= 0 80 0 0, clip]{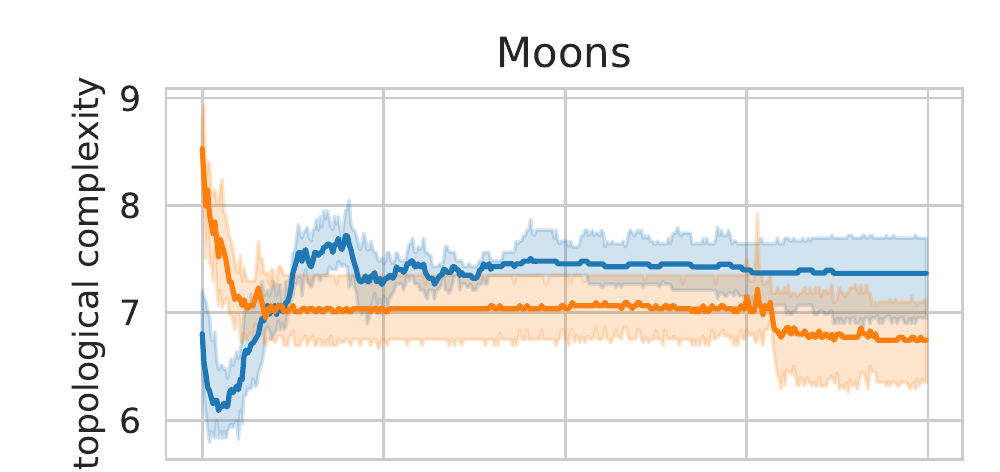}
    \includegraphics[width=0.33\columnwidth, trim= 0 80 0 0, clip]{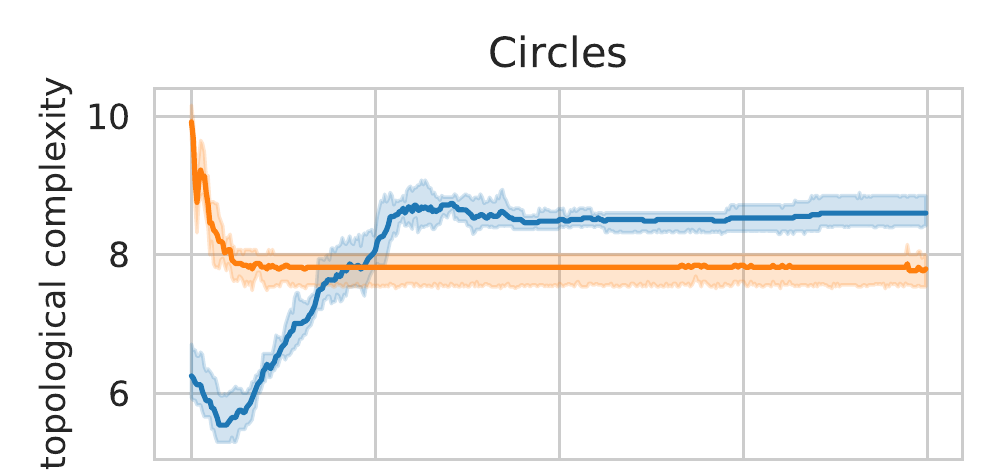}\\
    
    \includegraphics[width=0.33\columnwidth, trim= 0 30 0 0, clip]{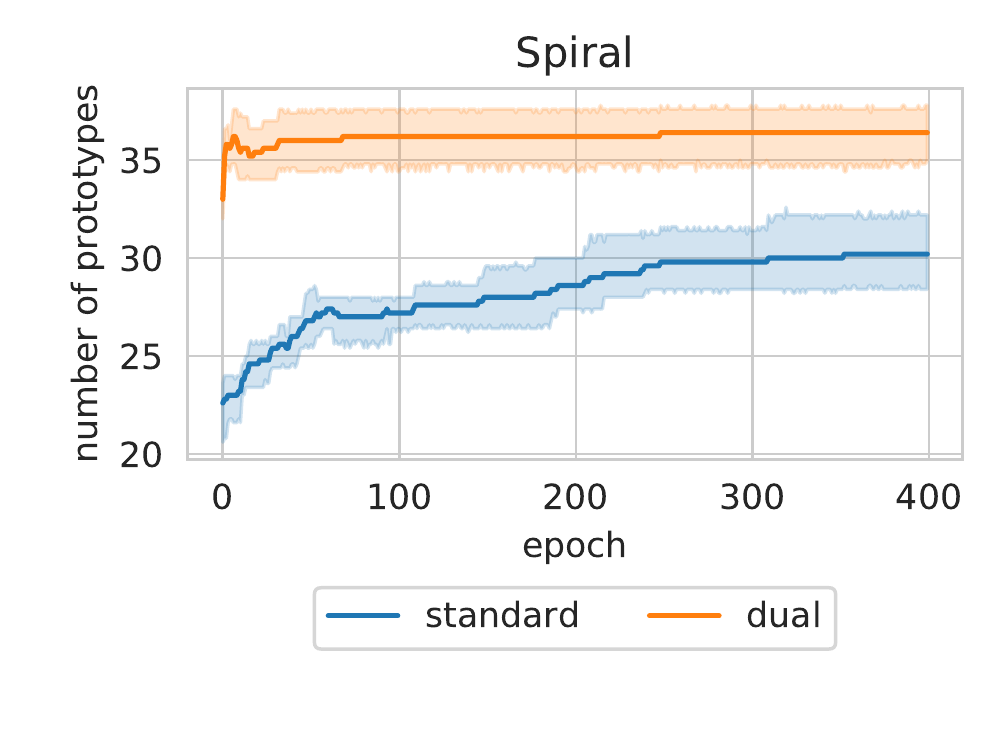}
    \includegraphics[width=0.33\columnwidth, trim= 0 30 0 0, clip]{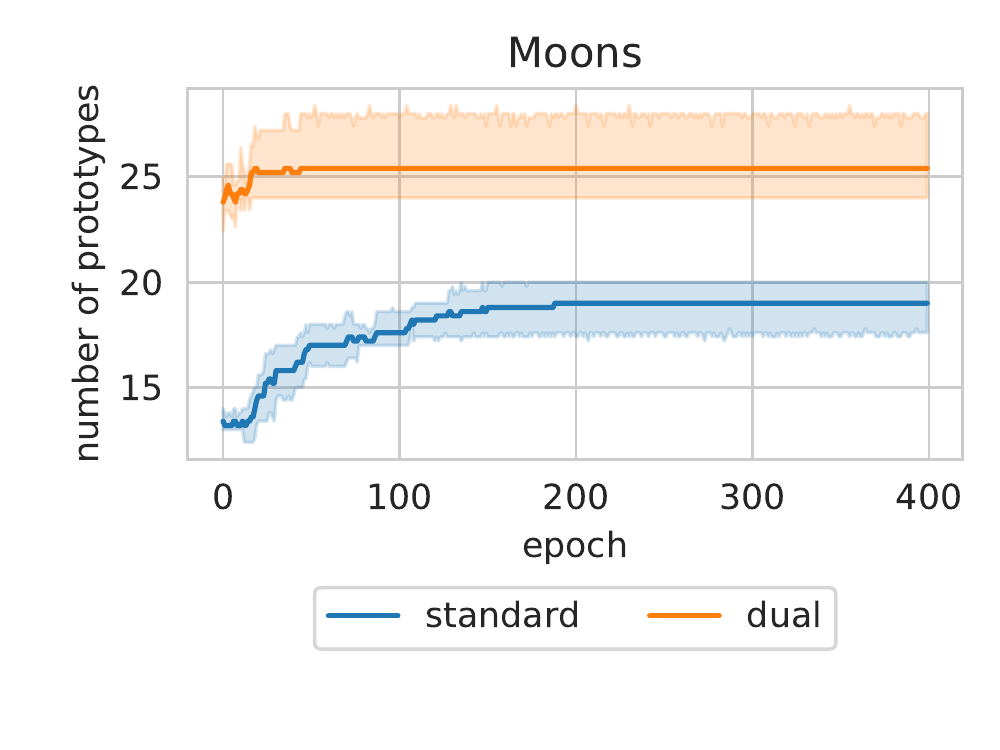}
    \includegraphics[width=0.33\columnwidth, trim= 0 30 0 0, clip]{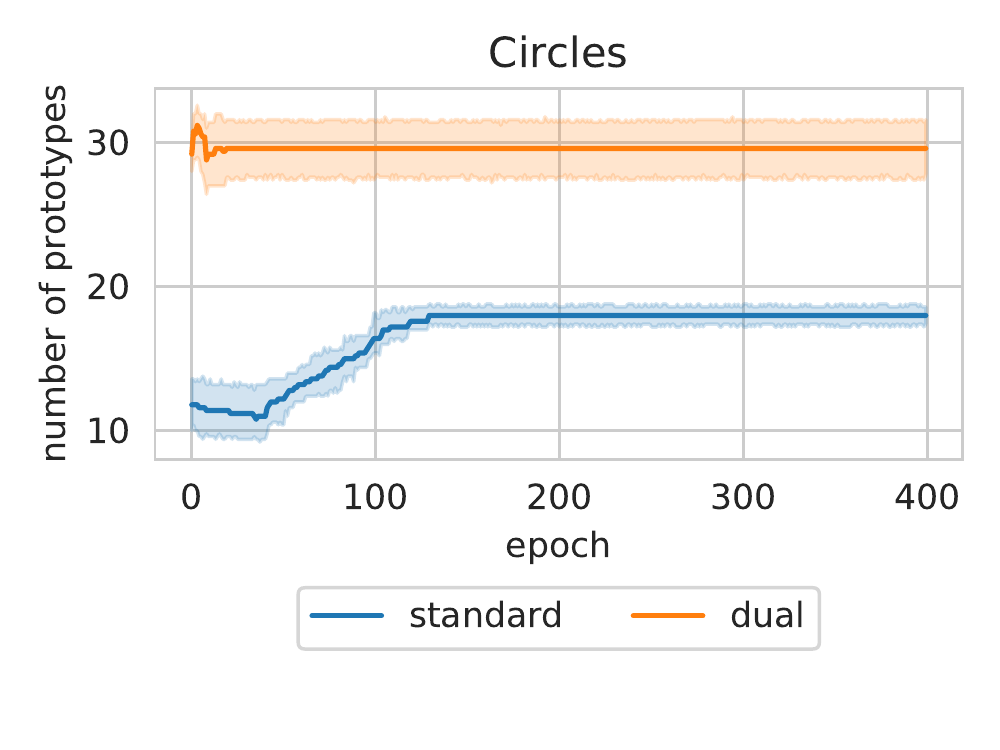}
    \caption{Comparison of three key metrics between the vanilla single-layer network and its dual over $10$ runs. The metrics are: the quantization error (\textbf{top row}), the norm of the matrix of the edges (\textbf{middle row}), and the number of valid prototypes (\textbf{bottom row}). The metrics are computed on three different datasets: \textit{Spiral} (\textbf{left column}), \textit{Moons} (\textbf{middle column}), and \textit{Circles} (\textbf{right column}). Error bands represent the standard error of the mean.}
    \label{fig:loss}
\end{figure*}
In order to rigorously assess the main characteristics of the learning process, several metrics are evaluated while training the two networks on three synthetic datasets (see Appendix \ref{app:1}). Figure \ref{fig:loss} shows for each dataset the dynamics of three key metrics for both VCL and DCL: the quantization error, the topological complexity of the solution (i.e. $||E||$), and the number of valid prototypes (i.e. the ones with a non-empty Voronoi set). By looking at the quantization error both networks tend to converge to similar local minima in all scenarios, thus validating their theoretical equivalence.
Nonetheless, the single-layer dual network exhibits a much faster rate of convergence compared to the vanilla competitive layer. The most significant differences are outlined (i) by the number of valid prototypes as DCL tends to employ more resources and (ii) by the topological complexity as VCL favors more complex solutions.
Fig. \ref{fig:exp1} shows topological clustering results after $800$ epochs (hyper-parameter settings are described in Appendix \ref{app:1}). As expected, both neural networks yield an adequate estimation of prototype positions, even though the topology learned by DCL is far more accurate in following the underlying manifolds w.r.t. to VCL.

\begin{figure}[!b]
    \centering
    \rotatebox{90}{$\ \ \ \ $\parbox{0.5cm}{\textsc{VCL}}}
    \includegraphics[width=0.28\columnwidth]{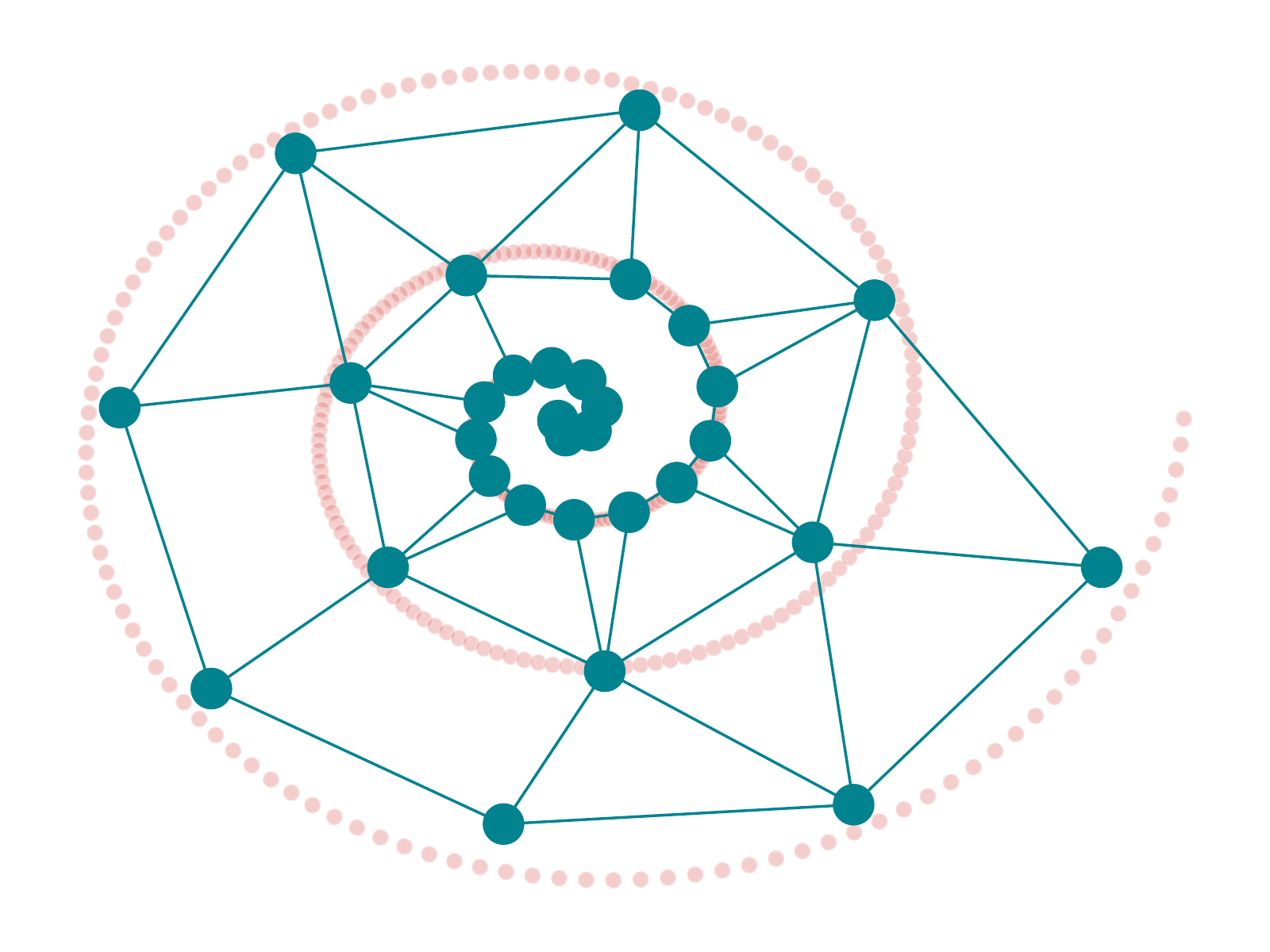}
    \includegraphics[width=0.28\columnwidth]{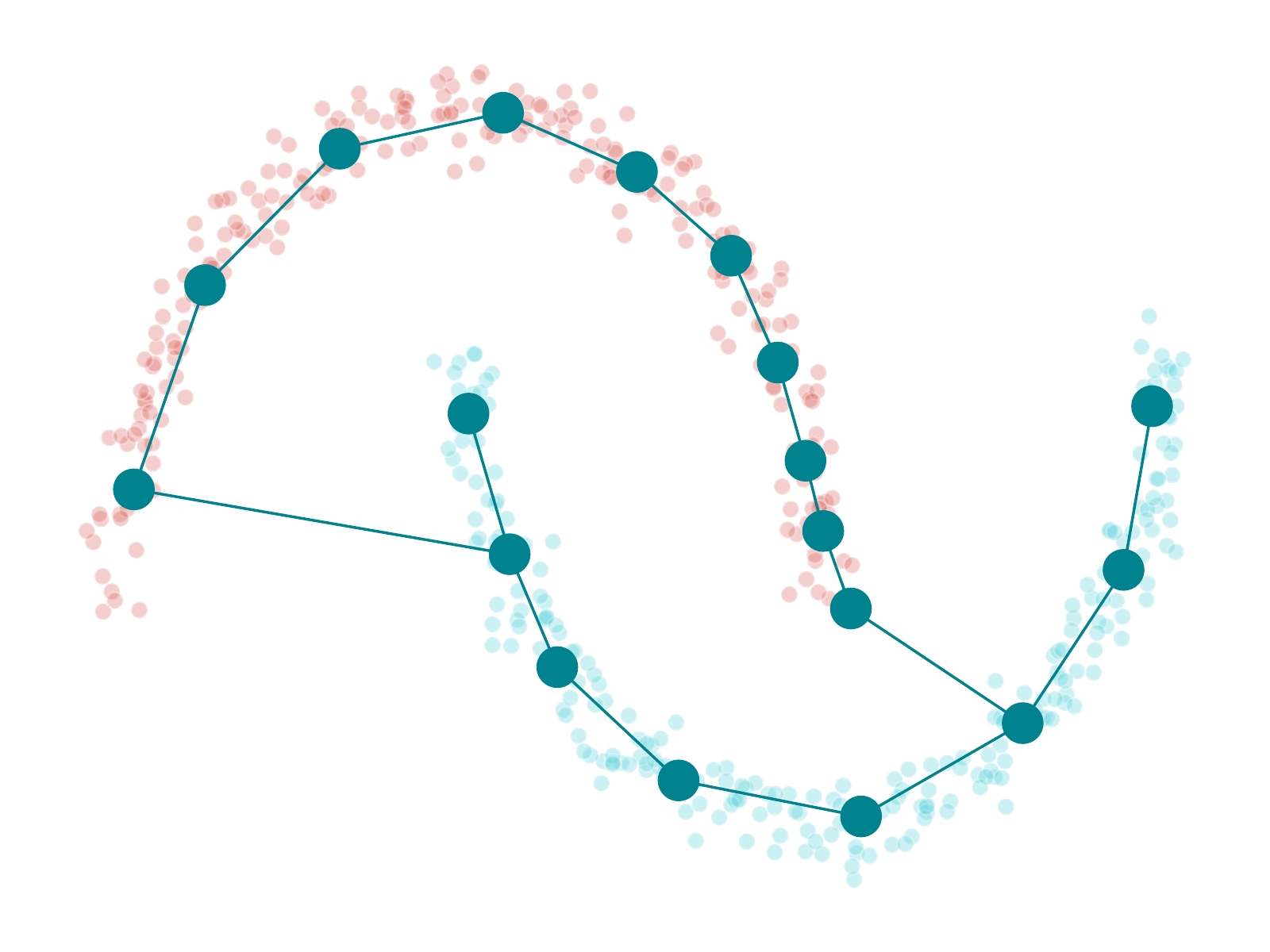}
    \includegraphics[width=0.28\columnwidth]{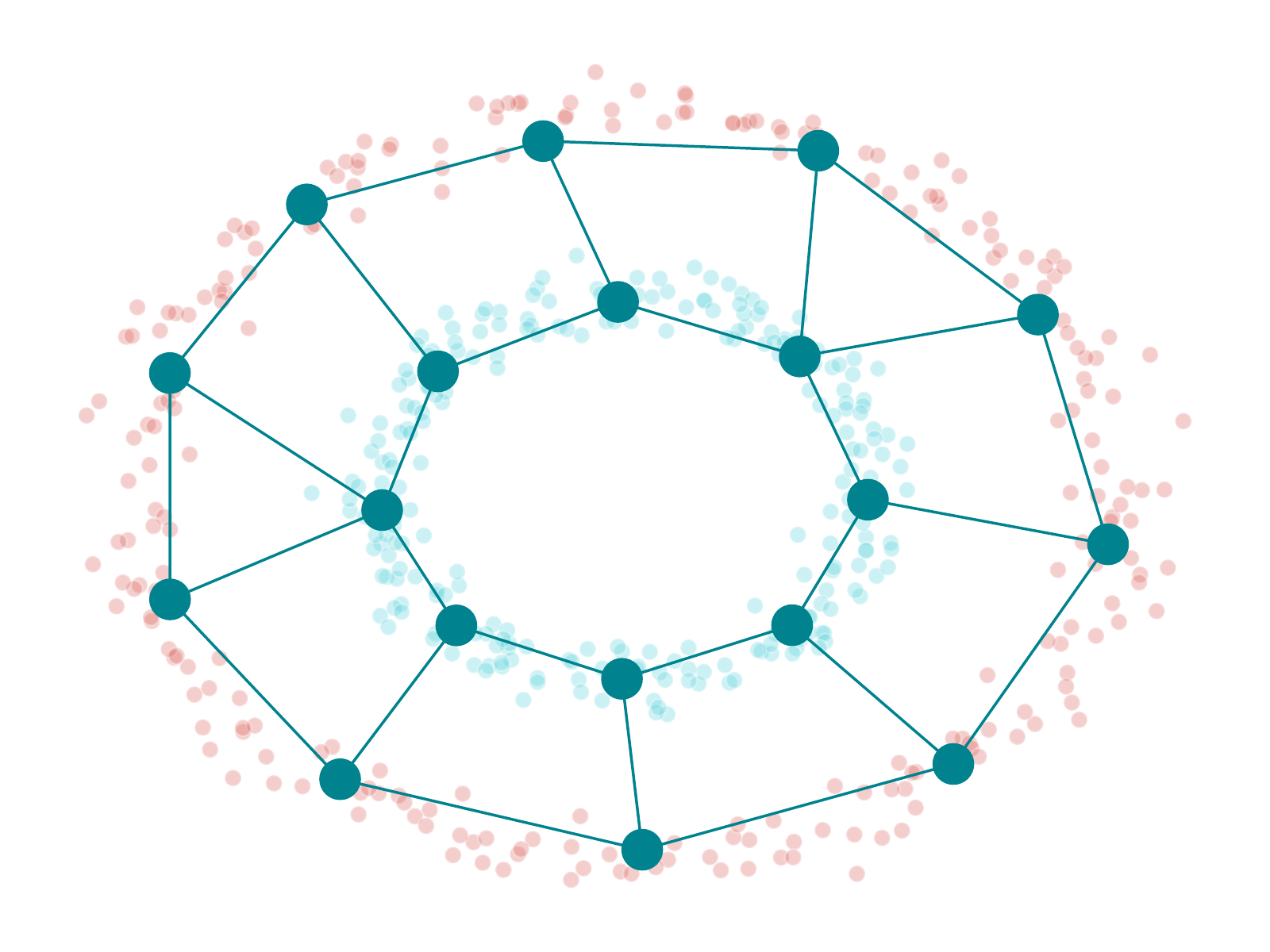}\\
    \vskip 0.5cm
    \rotatebox{90}{$\ \ \ \ $\parbox{0.5cm}{\textsc{DCL}}}
    \includegraphics[width=0.28\columnwidth]{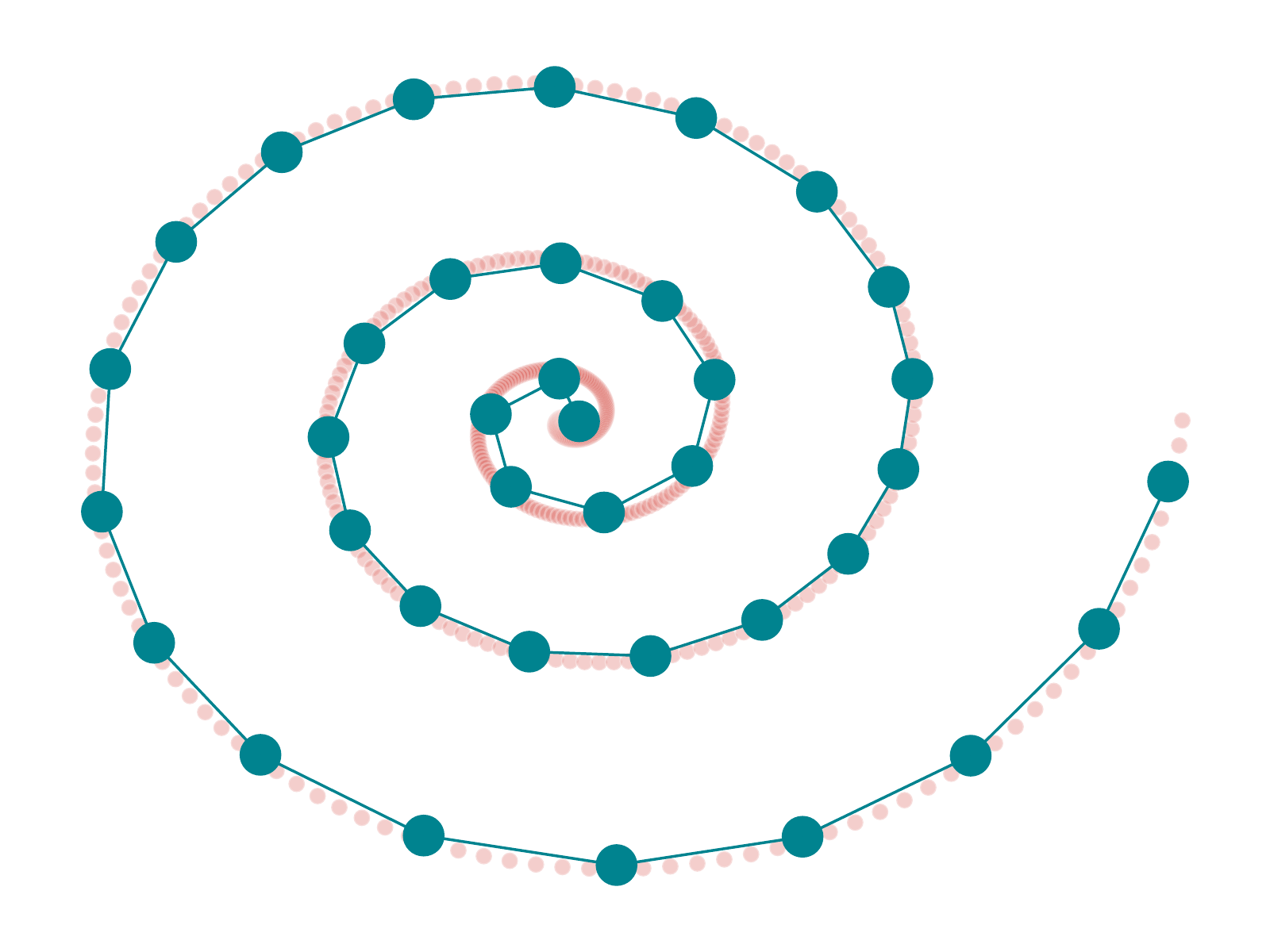}
    \includegraphics[width=0.28\columnwidth]{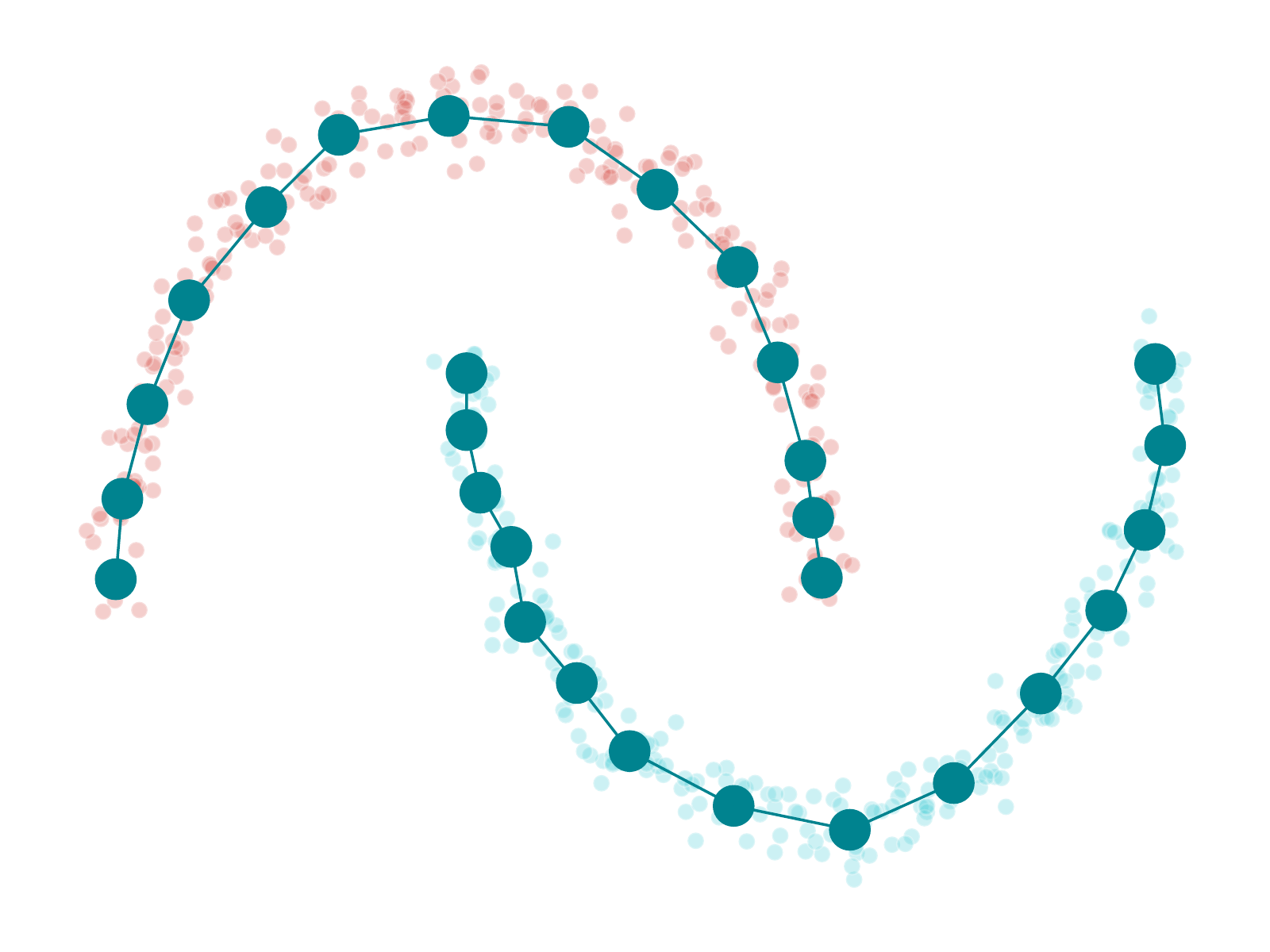}
    \includegraphics[width=0.28\columnwidth]{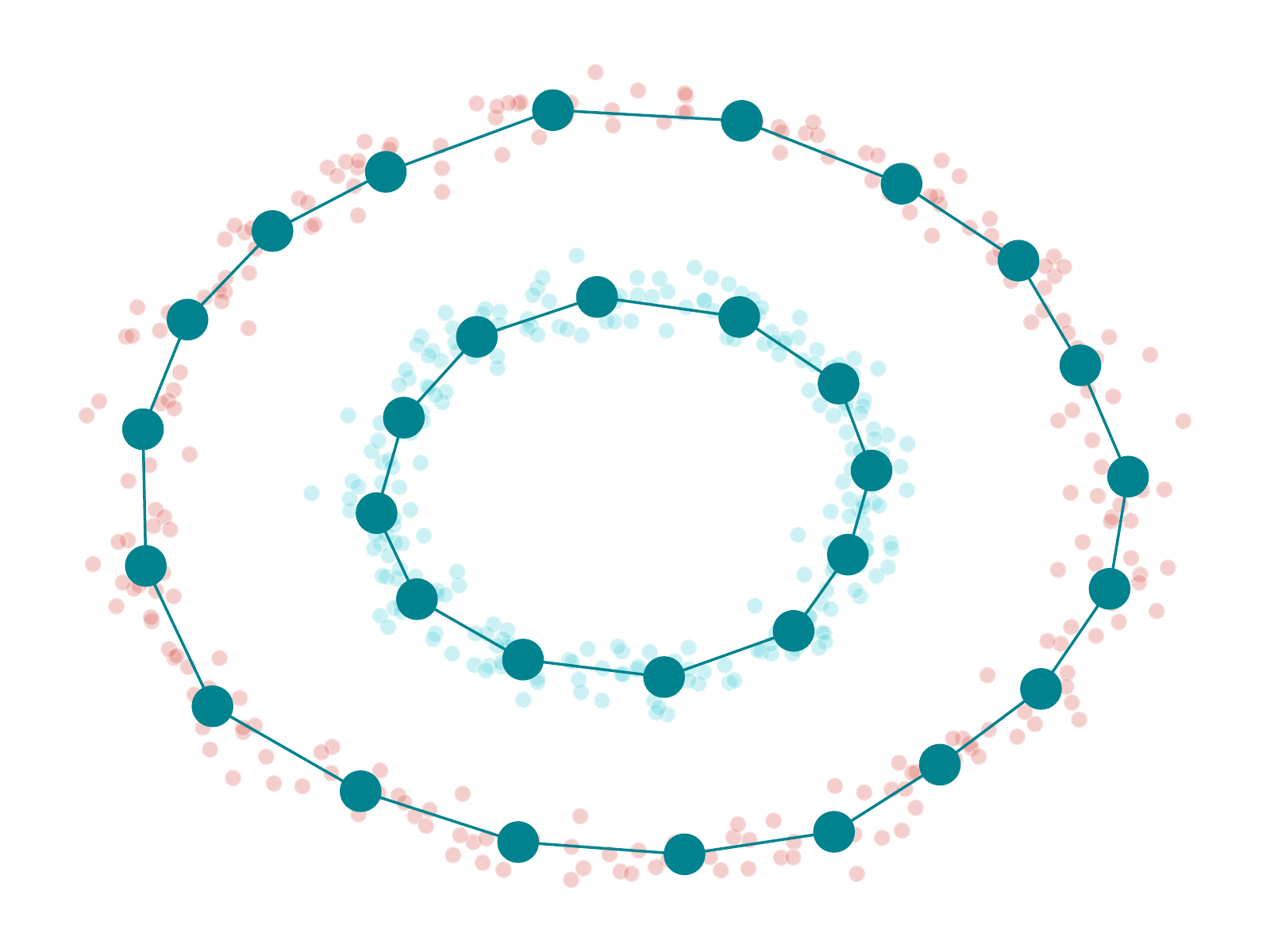}\\
    \caption{Experiments on synthetic datasets. From left to right: \textit{Spiral}, \textit{Moons}, and \textit{Circles} dataset.}
    \label{fig:exp1}
\end{figure}

\begin{figure}[!b]

    \centering
    \rotatebox{90}{$\ \ \ \ $\parbox{0.5cm}{\textsc{VCL}}}
    \includegraphics[width=0.28\columnwidth]{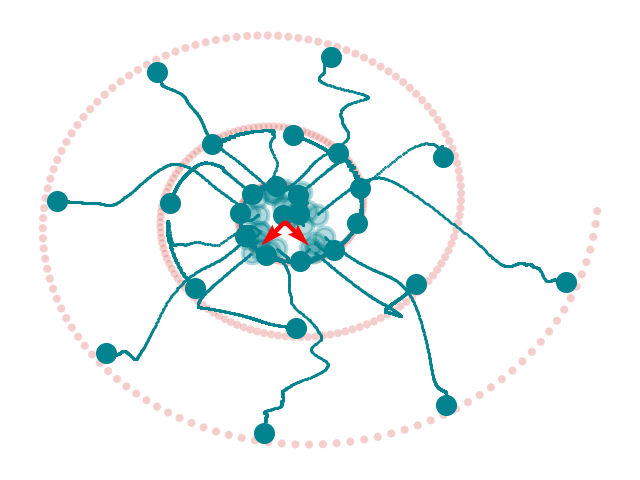}
    \includegraphics[width=0.28\columnwidth]{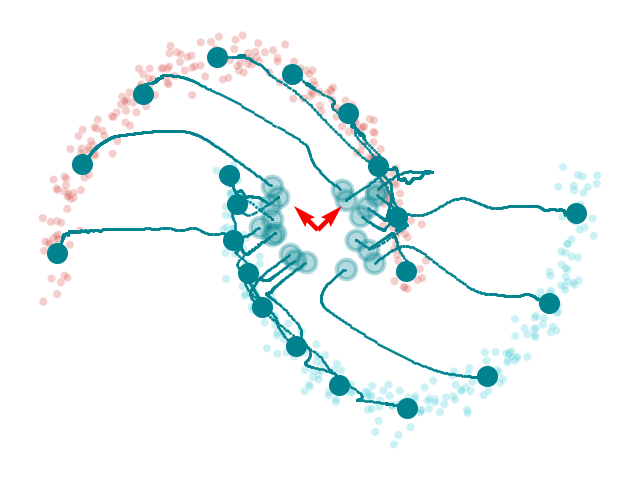}
    \includegraphics[width=0.28\columnwidth]{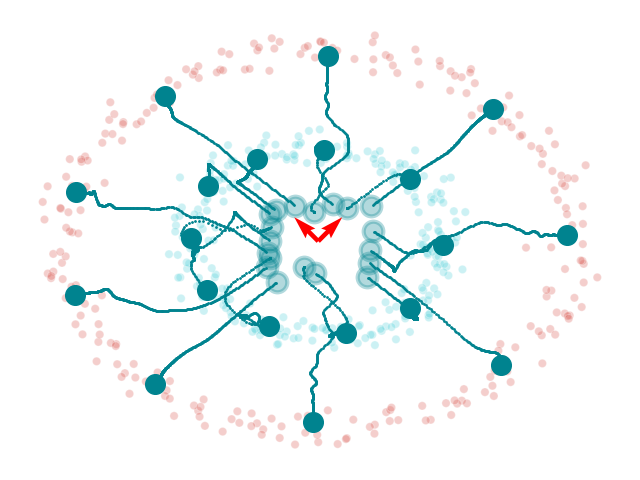}\\
    \vskip 0.5cm
    \rotatebox{90}{$\ \ \ \ $\parbox{0.5cm}{\textsc{DCL}}}
    \includegraphics[width=0.28\columnwidth]{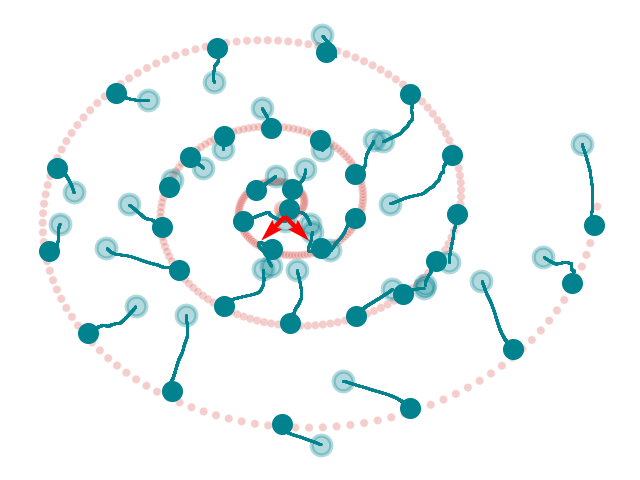}
    \includegraphics[width=0.28\columnwidth]{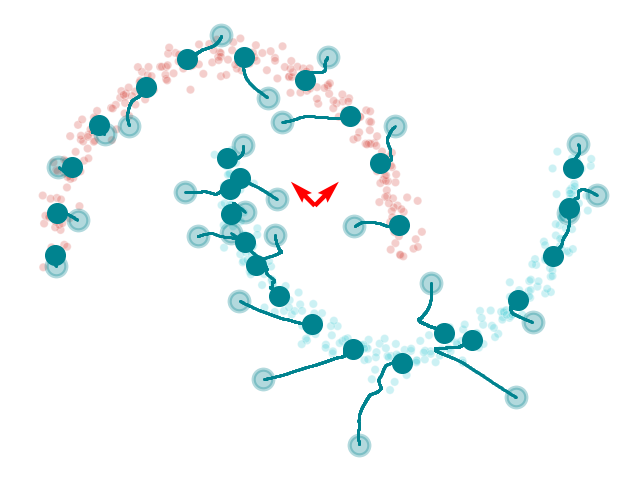}
    \includegraphics[width=0.28\columnwidth]{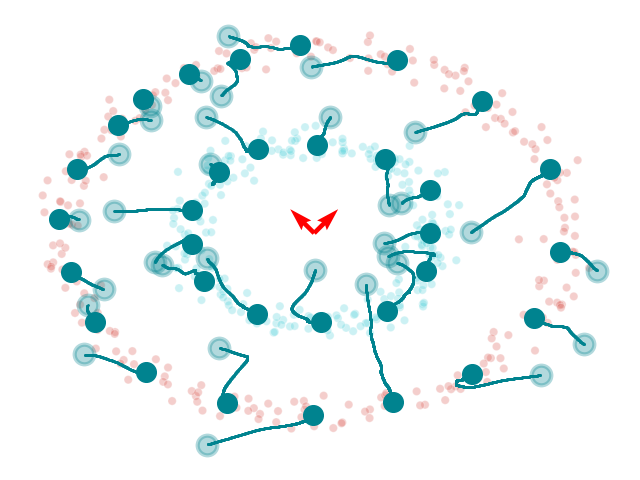}\\
    
    \caption{Dynamical simulations. From left to right: \textit{Spiral}, \textit{Moons}, and \textit{Circles} dataset.}
    \label{fig:exp2}
\end{figure}
Fig. \ref{fig:exp2} shows the trajectories of the prototypes during the training for both networks. The parameters in both networks have been initialized by means of the Glorot initializer \cite{glorot2010understanding}, which draws small values from a normal distribution centered on zero. For VCL, these parameters are the prototypes and they are initially clustered around the origin, as expected. For DCL, instead, the initial prototypes are an affine transformation of the inputs parameterized by the weight matrix. This implies the initial prototypes are close to a random choice of the input data. The VCL trajectories tend to the closest to the origin cluster and, then some of them spread towards the furthest manifolds. The DCL trajectories are much shorter because of the closeness of the initial prototypes to the input clusters. These considerations reveal the better suitability of DCL to deep learning traditional initializations.

Section \ref{sec:flows} yields a theoretical explanation for the observed results.

\section{Theoretical analysis}
\label{sec:flows}

\subsection{Stochastic approximation theory of the gradient flows}

In the following, the gradient flows of the vanilla and the dual single-layer neural networks are formally examined when trained using the quantization error, one of the most common loss functions used for training unsupervised neural networks in clustering contexts.
The following theory is based on the assumption of $\lambda=0$ in Eq. \ref{eq:loss}. Taking into account the edge error only relaxes the analysis, but the results remain valid.
Under the stochastic approximation theory, the asymptotic properties of the gradient flows of the two networks can be estimated.

\subsubsection{Base layer gradient flow}
For each prototype $j$, represented in the base layer by the weight vector $W_1^j \in \mathbb{R}^d$ of the $j$-th neuron (it is the $j$-th row of the matrix $W_1$), the contribution of its Voronoi set to the quantization error is given by:
\begin{equation}
    E^j = \sum_{i=1}^{n_j} \norm[\big]{x_i - W_1^j}_2^2 = \sum_{i=1}^{n_j} \big( \norm[\big]{x_i}_2^2 +  \norm[\big]{W_1^j}_2^2 - 2 x_i^T W_1^j \big)
\end{equation}
where $n_j$ is the cardinality of the $j$-th Voronoi set. The corresponding gradient flow of the base network is the following:
\begin{equation}
    W_1^j (t+1) = W_1^j (t) - \epsilon \nabla_{W_1^j} E^j = W_1^j (t) - \epsilon \sum_{i=1}^{n_j} \big( W_1^j - x_i \big)
\end{equation}
being $\epsilon$ the learning rate. The averaging ODE holds:
\begin{equation}
    \frac{dW_1^j}{dt} = - W_1^j + \mu_j
\end{equation}
where $\mu_j=\mathbb{E}[x_i]$ is the expectation in the limit of infinite samples of the $j$-th Voronoi set, and corresponds to the centroid of the Voronoi region. The unique critical point of the ODE is given by:
\begin{equation}
    W_{1,\textrm{crit}}^j = \mu_j
\end{equation}
and the ODE can be rewritten as:
\begin{equation}
    \frac{dw_1^j}{dt} = -w_1^j
\end{equation}
under the transformation $w_1^j = W_1^j - W_{1,\textrm{crit}}^j$ in order to study the origin as the critical point. The associated matrix is $-I_d$, whose eigenvalues are all equal to $-1$ and whose eigenvectors are the vectors of the standard basis. Hence, the gradient flow is stable and decreases in the same exponential way, as $e^{-t}$, in all directions. The gradient flow of one epoch corresponds to an approximation of the second step of the generalized Lloyd iteration, as stated before.

\subsubsection{Dual layer gradient flow}
In the dual layer, the prototypes are estimated by the outputs, in such a way that they are represented by the rows of the $Y_2$ matrix. Indeed, the $j$-th prototype is now represented by the row vector $(Y_2^j)^T$, from now on called $y_j^T$ for sake of simplicity. It is computed by the linear transformation:
\begin{equation}
    y_j^T = (W_2^j)^T [x_1 \ \cdots \ x_d] = (W_2^j)^T X^T = \Omega_j^T X^T
\end{equation}
where $x_i \in \mathbb{R}^n$ is the $i$-th row of the training set $X$ and $W_2^j \in \mathbb{R}^n$ is the weight vector of the $j$-th neuron (it is the $j$-th row of the matrix $W_2$), and is here named as $\Omega_j$ for simplicity. Hence, the $j$-th prototype is computed as:
\begin{equation}
    y_j = X \Omega_j
\end{equation}
and its squared (Euclidean) 2-norm is:
\begin{equation}
    \norm[\big]{y_j}_2^2 = \Omega_j^T X^T X \Omega_j
\end{equation}
For the $j$-th prototype, the contribution of its Voronoi set to the quantization error is given by:
\begin{equation}
    E^j = \sum_{i=1}^{n_j} \norm[\big]{x_i - y_1^j}_2^2 = \sum_{i=1}^{n_j} \big( \norm[\big]{x_i}_2^2 +  \norm[\big]{y_1^j}_2^2 - 2 x_i^T y_1^j \big)
\end{equation}
with the same notation as previously. The gradient flow of the dual network is computed as:
\begin{equation}
    \Omega_j (t+1) = \Omega_j (t) - \epsilon \nabla_{\Omega_j} E^j
\end{equation}
being $\epsilon$ the learning rate. The gradient is given by:
\begin{eqnarray}
    \nabla_{\Omega_j} E^j &=& \nabla_{\Omega_j} \sum_{i=1}^{n_j} (x_i^T x_i + \Omega_j^T X^T X \Omega_j - 2 \Omega_j^T X^T x_i) = \nonumber \\
    &=& 2 (X^T X \Omega_j - X^T x_i)
\end{eqnarray}
The averaging ODE is estimated as:
\begin{equation}
    \frac{d\Omega_j}{dt} = - \big(X^T X \Omega_j - X^T \mu_j \big) 
\end{equation}
The unique critical point of the ODE is the solution of the normal equations:
\begin{equation}
    X^T X \Omega_j = X^T \mu_j
\end{equation}
The linear system can be solved only if  $X^T X \in \mathbb{R}^{n\times n}$ is full rank. This is true only if $n \leq d$ (the case $n=d$ is trivial and, so, from now on the analysis deals with $n<d$) and all columns of $X$ are linearly independent. In this case, the solution is given by:
\begin{equation}
    \Omega_{j,\textrm{crit}} = \big( X^T X \big)^{-1} X^T \mu_j = X^+ \mu_j
\end{equation}
where $X^+$ is the pseudoinverse of $X$. The result corresponds to the least squares solution of the overdetermined linear system:
\begin{equation}
    X \Omega_j = \mu_j
\end{equation}

which is equivalent to:
\begin{equation}
    \Omega_j^T X^T = \mu_j^T
\end{equation}
This last system shows that the dual layer asymptotically tends to output the centroids as prototypes. The ODE can be rewritten as:
\begin{equation}
    \frac{dw_j}{dt} = -X^T X w_j
\end{equation}
under the transformation $w_j = \Omega_j - \Omega_{j,\textrm{crit}}$ in order to study the origin as the critical point. The associated matrix is $- X^T X$. Consider the singular value decomposition (SVD) of $X=U \Sigma V^T$ where $U \in \mathbb{R}^{d \times d}$ and $V \in \mathbb{R}^{n \times n}$ are orthogonal and $\Sigma \in \mathbb{R}^{d \times n)}$ is diagonal (nonzero diagonal elements named singular values and called $\sigma_i$, indexed in decreasing order). The $i$-th column of $V$ (associated to $\sigma_i$) is written as $v_i$ and is named right singular vector. Then:
\begin{equation} \label{eq:eigen-decomp}
    X^T X = \big( U \Sigma V^T \big)^T U \Sigma V^T = V \Sigma^2 V^T
\end{equation}
is the eigenvalue decomposition of the sample autocorrelation matrix of the inputs of the dual network. It follows that the algorithm is stable and the ODE solution is given by:
\begin{equation}
    w_j (t) = \sum_{i=1}^{n} c_i v_i e^{- \sigma_i^2 t}
\end{equation}
where the constants depend on the initial conditions. The same dynamical law is valid for all the other weight neurons.
If $n>d$ and all columns of $X$ are linearly independent, it follows:
\begin{equation}
    \textrm{rank} (X) = \textrm{rank} (X^T X) = d
\end{equation}
and the system $X \Omega_j = \mu_j$ is underdetermined. This set of equations has a nontrivial nullspace and so the least squares solution is not unique. However, the least squares solution of minimum norm is unique. This corresponds to the minimization problem:
\begin{equation}
    \min(\Omega_j) \qquad \textrm{s.t.} \   X \Omega_j = \mu_j
\end{equation}
The unique solution is given by the normal equations of the second kind:
\begin{equation}
    \begin{cases}
    X X^T z = \mu_j \\
    \Omega_j = X^T z
    \end{cases}
\end{equation}
that is, by considering that $X X^T$ has an inverse:
\begin{equation}
    \Omega_j = X^T \big( X X^T \big)^{-1} \mu_j
\end{equation}
Multiplying on the left by $X X^T$ yields:
\begin{eqnarray}
    \big( X^T X \big) \Omega_j &=& \big( X^T X \big) X^T \big( X^T X \big)^{-1} \mu_j = \nonumber \\
    &=& X^T \big( X^T X \big) \big( X^T X \big)^{-1} \mu_j = X^T \mu_j
\end{eqnarray}
that is Eq.(b), which is the system whose solution is the unique critical point of the ODE (setting the derivative of Eq.(a) to zero).
Resuming, both cases give the same solution. However, in the case $n>d$ and $\textrm{rank}(X)=d$, the output neuron weight vectors have minimum norm and are orthogonal to the nullspace of $X$, which is spanned by $v_{n-d+1)}, v_{n-d+2)}, \dots , v_n$. Indeed,  $X^T X$ has $n-d$ zero eigenvalues, which correspond to centers. Therefore, the ODE solution is given by:
\begin{equation}
    \label{eq:general-evo}
    w_j (t) = \sum_{i=1}^{n-d} c_i v_i e^{- \sigma_i^2 t} + \sum_{i=n-d+1}^{n} c_i v_i
\end{equation}
This theory proves the following theorem.

\begin{theorem}[Dual flow and PCA] \label{thm:dual-pca}
The dual network evolves in the directions of the principal axes of its autocorrelation matrix (see Eq. \ref{eq:eigen-decomp}) with time constants given by the inverses of the associated data variances.
\end{theorem}

This statement claims the dual gradient flow moves faster in the more relevant directions, i.e. where data vary more. Indeed, the trajectories start at the initial position of the prototypes (the constants in Eq. \ref{eq:general-evo} are the associated coordinates in the standard framework rotated by $V$) and evolve along the right singular vectors, faster in the directions of more variance in the data. It implies a faster rate of convergence because it is dictated by the data content, as already observed in the numerical experiments (see Fig. \ref{fig:loss}).

\subsection{Dynamics of the dual layers}
For the basic layer it holds:
\begin{equation}
    W_1^j - W_{1,\textrm{crit}}^j = l e^{-t}
\end{equation}
where $l \in \mathbb{R}^d$ is a vector of constants. Therefore, $W_1^j$ tends asymptotically to $\mu_j$, by moving in $\mathbb{R}^d$. However, being $\mu_j$ a linear combinations of the columns of $X$, it can be deduced that, after a transient period, the neuron weight vectors tend to the range (column space) of $X$, say $R(X)$, i.e.:
\begin{equation}
    \forall j, \forall t > t_0 \qquad W_1^j \in R(X) = \textrm{span}(u_1, u_2, \dots, u_r)
\end{equation}
where $t_0$ is a certain instant of time and $r=\textrm{rank}(X)=\min \{d,n\}$ under the assumption of samples independently drawn from the same distribution, which prevents from the presence of collinearities in data. It follows:
\begin{equation}
    W_1^j = W_{1,\textrm{crit}}^j + Uc e^{-t}
\end{equation}
where $l \in \mathbb{R}^d$ is another vector of constants. Then $W_1^j$ can be considered as the output of a linear transformation represented by the matrix $X$, i.e. $W_1^j = X p$, being $p \in \mathbb{R}^n$ its preimage. Hence, $(W_1^j)^T= p^T X^T$, which shows the duality. Indeed, it represents a network whose input is $X^T$, and the output $(W_1^j)^T$ and parameter weight vector $p^T$ are the interchange of the corresponding ones in the base network. Notice, however, that the weight vector in the dual network corresponds only through a linear transformation, that is, by means of the preimage.
Under the second duality assumption $XX^T = I_d$, it holds:
\begin{eqnarray}
    & XX^T = U \Sigma V^T (U \Sigma V^T)^T = U \Sigma \Sigma^T U^T = I_d \nonumber \\
    &\implies U \Sigma \Sigma^T = U \nonumber \\
    &\implies
    \begin{dcases}
    U I_d = U \quad & d \leq n \\
    U 
    \begin{bmatrix}
    I_n & 0_{n,d-n} \\
    0_{d-n,n} & 0_{d-n,d-n}
    \end{bmatrix} = U \quad & d > n
    \end{dcases}
\end{eqnarray}
where $0_{r,s}$ is the zero matrix with $r$ rows and $s$ columns. Therefore, this assumption implies there are $d$ singular values all equal to $1$ or $-1$. In case of remaining singular values, they are all null and of cardinality $d-n$. 
For the dual layer, under the second duality assumption, in the case of singular values all equal to $-1$ or $0$, it follows:
\begin{equation}
    \Omega_j - \Omega_{j,\textrm{crit}} = 
    \begin{cases}
    Vqe^{-t} \quad & d \geq n \\
    Vq 
    \begin{bmatrix}
    e^{-t} 1_d \\
    1_{n-d}
    \end{bmatrix} \quad & d < n
    \end{cases}
\end{equation}
where $q \in \mathbb{R}^n$. Therefore, $\Omega_j$ tends asymptotically to $\Omega_{j,\textrm{crit}}$, by moving in $\mathbb{R}^n$. Hence, it can be deduced that, after a transient period, the neuron weight vectors tend to the range (column space) of $X$, say $R(X^T)$, i.e.:
\begin{equation}
    \forall j, \forall t > t_0 \qquad \Omega^j \in R(X^T) = \textrm{span}(v_1, v_2, \dots, v_r)
\end{equation}
where $t_0$ is a certain instant of time and $r=\textrm{rank}(X)=\min \{d,n\}$ under the same assumption of noncollinear data.

Resuming, the base and dual gradient flows, under the two duality assumptions, except for the presence of centers, are given by:
\begin{eqnarray} \label{eq:gradient-flows}
    & \begin{cases}
    w_1^j = U c e^{-t} \\
    \omega_j = V q e^{-t}
    \end{cases} \implies \nonumber \\
    & X \omega_j = X V q e^{-t} \implies \nonumber 
    X\omega_j = U\Sigma q e^{-t} \implies \nonumber\\ 
    & X\omega_j = U c e^{-t} \implies \nonumber 
    w_1^j = X\omega_j
\end{eqnarray}
because $XV=U\Sigma$ from the SVD of $X$ and $c=\Sigma q$ for the arbitrariness of the constants. This result claims the fact that the base flow directly estimates the prototype, while the dual flow estimates its preimage. This confirms the duality of the two layers from the dynamical point of view and proves the following theorem.

\begin{theorem}[Dynamical duality]
Under the two assumptions of \ref{thm:duality}, the two networks are dynamically equivalent. In particular, the base gradient flow evolves in $R(X)$ and the dual gradient flow evolves in $R(X^T)$.
\end{theorem}

More in general, the fact that the prototypes are straightly computed in the base network implies a more rigid dynamics of its gradient flow. On the contrary, the presence of the singular values in the exponentials of the dual gradient flow originate from the fixed transformation (matrix $X$) used for the prototype estimation (see Eq. \ref{eq:gradient-flows}). They are exploited for a better dynamics, because they are suited to the statistical characteristics of the training set, as discussed before.
Both flows estimate the centroids of the Voronoi sets, like the centroid estimation step of the Lloyd algorithm, but the linear layers allow the use of gradient flows and do not require the a priori knowledge of the number of prototypes (see the discussion on pruning in Section \ref{sec:analysis}). However, the dual flow is an iterative least squares solution, while the base flow does the same only implicitly.
In the case $d>n$, $\textrm{rank}(X) = \textrm{rank}(X^T)=n$, and the base gradient flow stays in $\mathbb{R}^d$, but tends to lie on the $n$-dimensional subspace $R(X)$. Instead, the dual gradient flow is $n$-dimensional and always evolves in the $n$-dimensional subspace $R(X^T)$. Figure \ref{fig:flows-subspaces} shows both flows and the associated subspaces for the case $n=2$ and $d=3$. The following lemma describes the relationship between the two subspaces.

\begin{figure}
    \centering
    \includegraphics[scale=0.4,trim={0cm 3cm 0cm 1cm},clip]{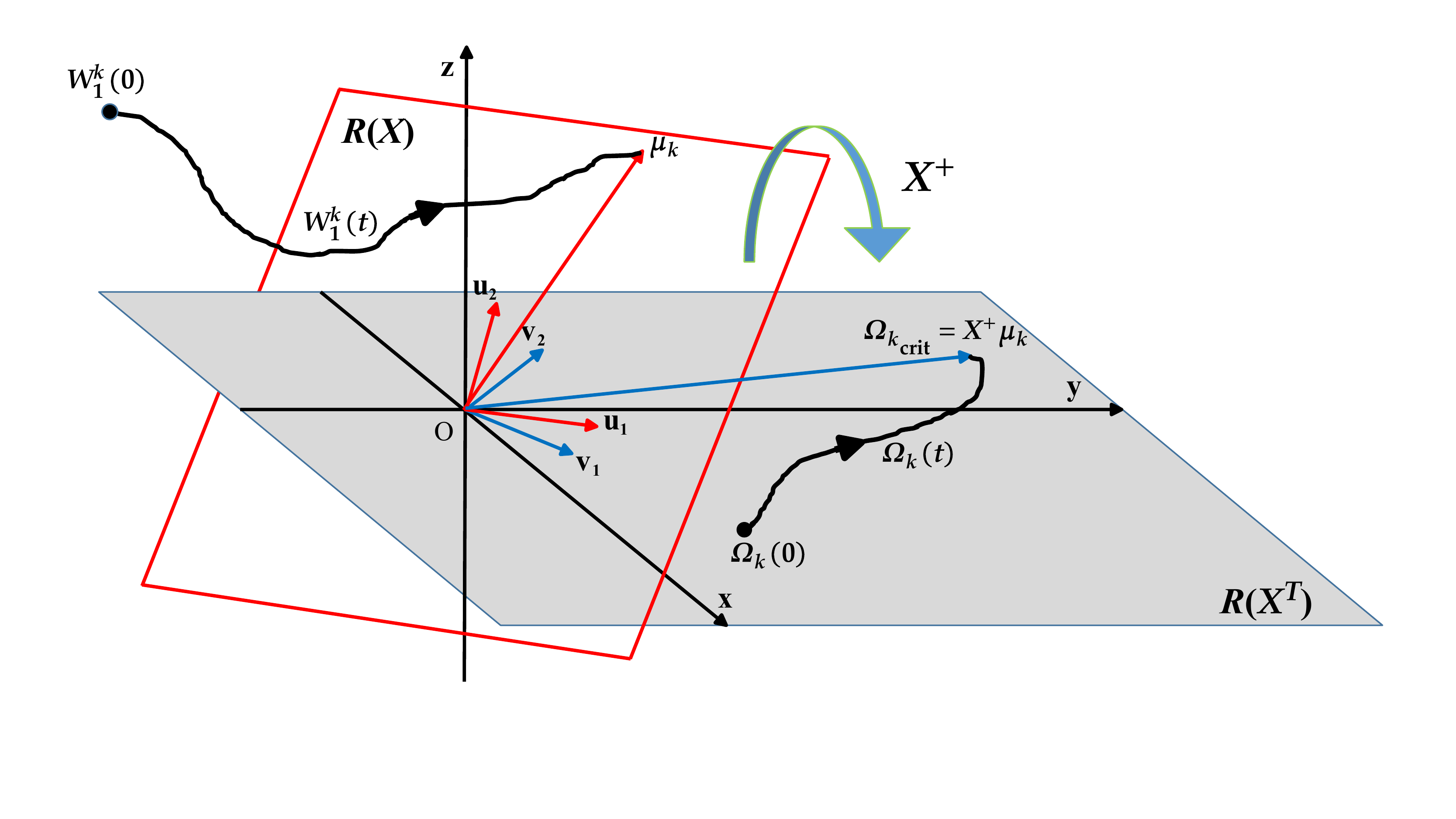}
    \caption{Gradient flows and subspaces ($n=2$ and $d=3$).}
    \label{fig:flows-subspaces}
\end{figure}

\begin{lemma}[Range transformation]
The subspace $R(X)$ is the transformation by $X$ of the subspace $R(X^T)$.
\end{lemma}

\begin{proof}
The two subspaces are the range (column space) of the two matrices $X$ and $X^T$:
\begin{equation}
    R(X) = \{z:\ z=Xu \quad \text{for a certain } u\}
\end{equation}
\begin{equation}
    R(X^T) = \{y:\ y=X^Tx \quad \text{for a certain } x\}
\end{equation}
Then:
\begin{equation}
    X R(X^T) = \{u=Xy:\ y=X^Tx \quad \text{for a certain } x\} = R(X)
\end{equation}
More in general, multiplying $X$ by a vector yields a vector in $R(X)$.
\end{proof}
All vectors in $R(X^T)$ are transformed by $X$ in the corresponding quantities in $R(X)$. In particular:
\begin{equation}
    u_i = \frac{1}{\sigma_i} X v_i \qquad \forall i = 1, \dots, n
\end{equation}
\begin{equation}
    \mu_j = X \Omega_{j,\textrm{crit}}
\end{equation}
\begin{figure}
    \centering
    \includegraphics[trim= 10 22 10 25, clip, width=0.7\columnwidth]{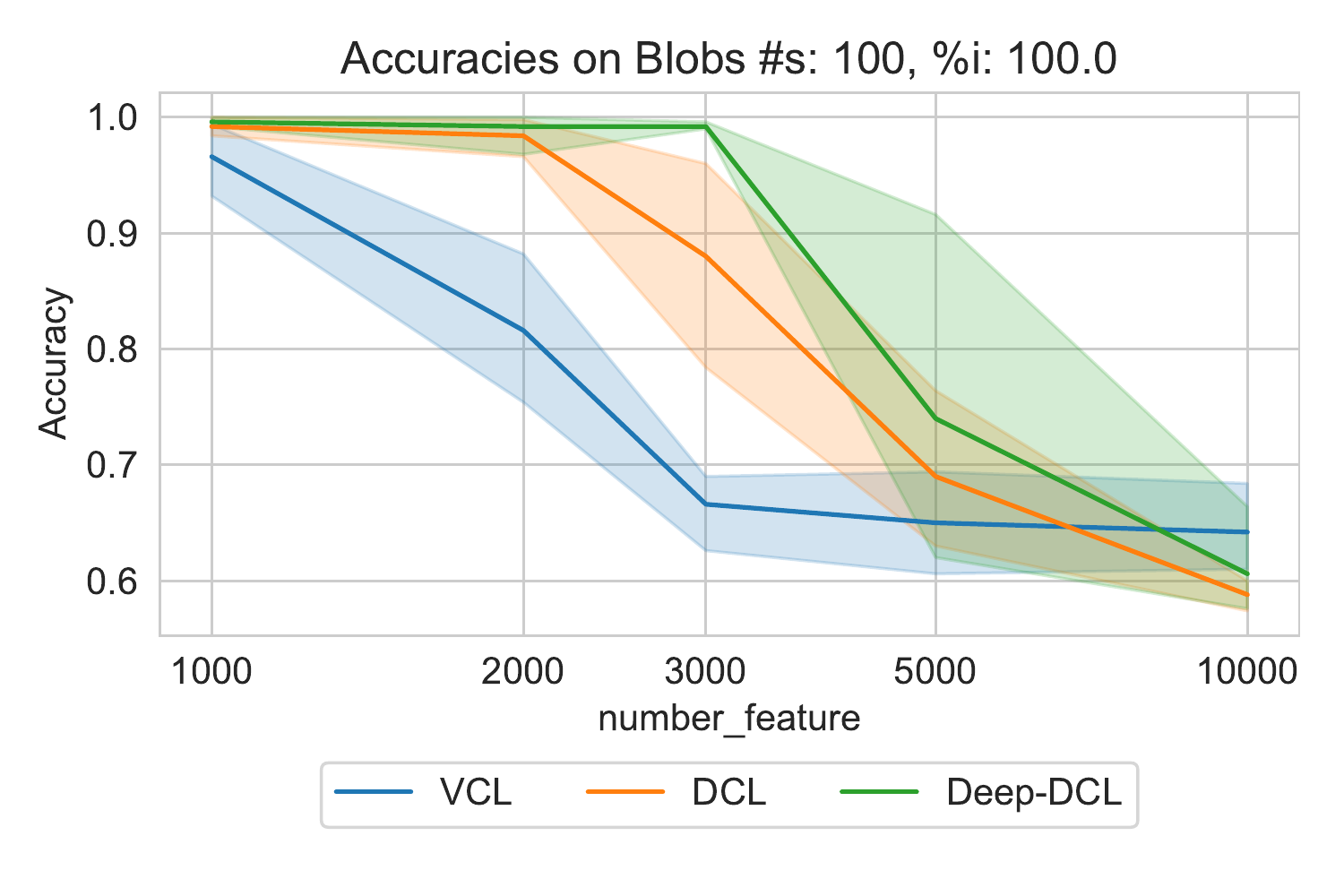}
    \caption{High-dimensional simulations: accuracy as a function of the dimensionality of the problem. Error bands correspond to the standard error of the mean.}
    \label{fig:high-dim}
\end{figure}
This analysis proves the following theorem. 
\begin{theorem}[Fundamental on gradient flows, part I]
In the case $d>n$, the base gradient flow represents the temporal law of a $d$-dimensional vector tending to an $n$-dimensional subspace containing the solution. Instead, the dual gradient flow always remains in an $n$-dimensional subspace containing the solution. Then, the least squares transformation $X^+$ yields a new approach, the dual one, which is not influenced by $d$, i.e. the dimensionality of the input data.
\end{theorem}
This assertion is the basis of the claim the dual network is a novel and very promising technique for high-dimensional clustering. However, it must be considered that the underlying theory is only approximated and gives an average behavior. Fig. \ref{fig:high-dim} shows a simulation comparing the performances of VCL, DCL, and a deep variant of the DCL model in tackling high-dimensional problems with an increasing number of features (see Appendix \ref{app:2} for simulation details). The simulations show how the dual methods are more capable to deal with high-dimensional data as their accuracy remains near $100\%$ until $2000-3000$ features. Obviously, the deep version of DCL (deep-DCL) yields the best accuracy because it exploits the nonlinear transformation of the additional layers.

In the case $n \geq d$, instead, the two subspaces have dimension equal to $d$. Then, they coincide with the feature space, eliminating any difference between the two gradient flows. In reality, for the dual flow, there are $n-d$ remaining modes with zero eigenvalue (centers) which are meaningless, because they only add $n-d$ constant vectors (the right singular vectors of $X$) which can be eliminated by adding a bias to each output neuron of the dual layer.

\begin{theorem}[Fundamental on gradient flows, part II]
In the case $d \leq n$, both gradient flows lie in the same (feature) space, the only difference being the fact that the dual gradient flow temporal law is driven by the variances of the input data.
\end{theorem} 

\subsection{The Voronoi set estimation}

Consider the matrix $X^T Y^T \in \mathbb{R}^{n \times j}$, which contains all the inner products between data and prototypes. From the architecture and notation of the dual layer, it follows  $Y=\Omega X^T$, which yields:
\begin{equation}
    X^T Y^T = X^T X \Omega^T = G\Omega^T
\end{equation}
where the sample autocorrelation data matrix $G$ is the Gram matrix.
The Euclidean distance matrix $\textrm{edm}(X,Y) \in \mathbb{R}^{n \times j}$, which contains the squared distances between the columns of $X$ and $Y$, i.e. between data and prototypes, is given by \cite{dokmanic2015euclidean}:
\begin{equation}
    \textrm{edm}(X,Y) = \textrm{diag}(X^T X) 1_j^T - 2 X^T Y^T + 1_n \textrm{diag}(YY^T)^T
\end{equation}
where $\textrm{diag}(A)$ is a column vector containing the diagonal entries of $A$ and $1_r$ is the $r$-dimensional column vector of all ones. It follows:
\begin{equation}
    \textrm{edm}(X,Y)=\textrm{diag}(G) 1_j^T - 2G \Omega^T + 1_n \textrm{diag}(YY^T)^T
\end{equation}
and, considering that $YY^T =  \Omega X^T (\Omega X^T)^T = \Omega G \Omega^T$, it holds:
\begin{eqnarray} \label{eq:edm}
    \textrm{edm}(X,Y) &=& f(G,\Omega) = \\
    &=& 1_n \textrm{diag}(\Omega G \Omega^T)^T - 2G \Omega^T + \textrm{diag}(G) 1_j^T \nonumber
\end{eqnarray}
as a quadratic function of the dual weights.
This function allows the straight computation of the edm from the estimated weights, which is necessary in order to evaluate the Voronoi sets of the prototypes for the quantization loss.



\section{Conclusion}
This work opens a novel field in neural network research where unsupervised gradient-based learning joins competitive learning.
Two novel layers (VCL and DCL) suitable for unsupervised deep learning applications are introduced. Despite VCL is just an adaptation of a standard competitive layer for deep neural architectures, DCL represents a completely novel approach. The relationship between the two layers has been extensively analyzed and their equivalence in terms of architecture has been proven. Nonetheless, the advantages of the dual approach justify its employment. Unlike all other clustering techniques, the parameters of DCL evolve in a $n$-dimensional submanifold which does not depend on the number of features $d$ as the layer is trained on the transposed input matrix. As a result, the dual approach is natively suitable for tackling high-dimensional problems. The flexibility and the power of the approach pave the way towards more advanced and challenging learning tasks; an upcoming paper will compare DCL on renowned benchmarks against state-of-the-art clustering algorithms. Further extensions of this approach may include topological nonstationary clustering \cite{randazzo2018nonstationary}, hierarchical clustering \cite{ghng,cirrincione2020gh}, core set discovery \cite{barbiero2020uncovering,ciravegna2019discovering}, incremental and attention-based approaches, or the integration within complex architectures such as VAEs and GANs, and will be studied in the future.


\printbibliography

\clearpage
\newpage

\appendix

\section{Appendix}
\subsection{Topological and dynamical simulations}
\label{app:1}
In order to validate the theory and to analyze the differences of the two learning approaches, VCL and DCL are compared on three synthetic datasets containing clusters of different shapes and sizes.
Table \ref{tab:datasets} summarizes the main characteristics of each experiment.
The first dataset is composed of samples drawn from a two-dimensional Archimedean spiral (\textit{Spiral}). The second dataset consists of samples drawn from two half semicircles (\textit{Moons}). The last one is composed of two concentric circles (\textit{Circles}).
Each dataset is normalized by removing the mean and scaling to unit variance before fitting neural models. For all the experiments, the number of output units $k$ of the dual network is set to $30$. A grid-search optimization is conducted for tuning the hyperparameters. The
learning rate is set to $\epsilon=0.008$ for VCL and to $\epsilon=0.0008$ for DCL. Besides, for both networks, the number of epochs is equal to $\eta=400$ while the Lagrangian multiplier to $\lambda=0.01$.
For each dataset, both networks are trained $10$ times using different initialization seeds in order to statistically compare their performance.

\begin{table}[!ht]
\renewcommand{\arraystretch}{1.5}
\centering
\caption{Synthetic datasets used for the simulations (s.v. stands for singular value).}
\label{tab:datasets}
\begin{center}
\begin{sc}
\begin{tabular}{@{}lrrrrr@{}}
\toprule
dataset & samples & features & clusters & max s.v. & min s.v. \\ \midrule
Spiral & 500 & 2 & 1 & 23.43 & 21.24 \\
Moons & 500 & 2 & 2 & 26.97 & 16.51 \\
Circles & 500 & 2 & 2 & 22.39 & 22.34 \\
\bottomrule
\end{tabular}%
\end{sc}
\end{center}
\vskip -0.1in
\end{table}

\subsection{High-dimensional simulations}
\label{app:2}
The performance of the vanilla competitive layer and its dual network in tackling high dimensional problems is assessed through numerical experiments. Sure enough, standard distance-based algorithms generally suffer the well-known curse of dimensionality when dealing with high-dimensional data.
The MADELON algorithm proposed in \cite{guyon2003design} is used to generate high-dimensional datasets with an increasing number of features and fixed number of samples. This algorithm creates clusters of points normally distributed about vertices of an $n$-dimensional hypercube. An equal number of cluster and data is assigned to two different classes. Both the number of samples ($n_s$) and the dimensionality of the space ($n_f$) in which they are placed can be defined programmatically. 
More precisely, the number of samples is set to $n_s=100$ while the number of features ranges in $n_f \in [1000, 2000, 3000, 5000, 10000]$. 
The number of required centroids is fixed to one tenth the number of input samples. Three different networks are compared: VCL, DCL, and a deep variant of DCL with two hidden layers of 10 neurons each (deep-DCL). Results are averaged over $10$ repetitions on each dataset. Accuracy for each cluster is calculated by considering true positive those samples belonging to the class more represented and false positive the remaining data.

\begin{table}[!ht]
\renewcommand{\arraystretch}{1.5}
\centering
\caption{Parameters for high-dimensional simulations using MADELON (s.v. stands for singular value).}
\label{tab:datasets}
\begin{center}
\begin{sc}
\begin{tabular}{@{}lrrrr@{}}
\toprule
samples & features & clusters & max s.v. & min s.v. \\ \midrule
100 & 1000 & 2 & 112 & 3e-14 \\
100 & 2000 & 2 & 120 & 7e-14 \\
100 & 3000 & 2 & 126 & 4e-14 \\
100 & 5000 & 2 & 139 & 5e-14 \\
100 & 10000 & 2 & 154 & 7e-14 \\
\bottomrule
\end{tabular}%
\end{sc}
\end{center}
\vskip -0.1in
\end{table}

\subsection{Software}
All the code for the experiments has been implemented in Python 3, relying upon open-source libraries \cite{abadi2016tensorflow,pedregosa2011scikit}.
All the experiments have been run on the same machine: Intel\textsuperscript{\textregistered} Core\texttrademark\ i7-8750H 6-Core Processor at 2.20 GHz equipped with 8 GiB RAM.

To enable code reuse, the Python code for the mathematical models including parameter values and documentation is freely available under Apache 2.0 Public License from a GitHub repository\footnote{\url{https://github.com/pietrobarbiero/cola}}\cite{barbiero_cola}. The whole package can also be downloaded directly from PyPI\footnote{\url{https://pypi.org/project/deeptl/1.0.0/}}.
Unless required by applicable law or agreed to in writing, software is distributed on an "as is" basis, without
warranties or conditions of any kind, either express or implied.

\end{document}